\documentclass[twoside,11pt]{article}

\usepackage{jmlr2e}

\jmlrheading{1}{2012}{1-2}{9/12}{-}{Feldman et al.}

\ShortHeadings{Multi-Task Averaging}{Feldman et al.}
\firstpageno{1}

\usepackage{amssymb,amsfonts}
\usepackage[cmex10]{amsmath}
\usepackage{epsfig,graphicx}
\usepackage{cite}

\def\argmin{\mathop{\rm arg \, min}}

\providecommand{\tr}{\textbf{tr}}

\providecommand{\one}{\textnormal{\textbf{1}}}
\providecommand{\zero}{\textnormal{\textbf{0}}}
\def\H{{\mathcal H}}
\def\R{{\mathbb R}}

\def\MSE{{\mathrm{MSE}}}

\begin{document}

\title{Multi-Task Averaging}

\author{\name Sergey Feldman \email sergeyf@u.washington.edu \\
 \addr Dept. of Electrical Engineering\\
       University of Washington\\
       Seattle, WA 98195, USA\\
       \name Maya R. Gupta \email mayagupta@google.com\\
       \addr Google Research\\
       \addr Google\\
       1225 Charleston Rd \\
       Mountain View, CA 94301, USA\\
       \name Bela A. Frigyik \email frigyik@gmail.com\\
       \addr Institute of Mathematics and Informatics\\
       University of P\'{e}cs \\
H-7624 P\'{e}cs, Ifj\'{u}s\'{a}g St. 6, Hungary\\
       }

\editor{}

\maketitle

\begin{abstract}%
We present a multi-task learning approach
to jointly estimate the means of multiple independent data sets.
The proposed multi-task averaging (MTA) algorithm results in
a convex combination of the single-task maximum
likelihood estimates. We derive the optimal minimum risk estimator and the minimax estimator,
and show that these estimators can be efficiently estimated.
Simulations and real data experiments
demonstrate that MTA estimators often outperform both single-task and James-Stein estimators.
\end{abstract}


\section{Introduction}
\label{sec:intro}
The motivating hypothesis behind multi-task learning (MTL) algorithms
is that leveraging data from related tasks can yield superior
performance over learning from each task independently. Early evidence for this hypothesis is Stein's work on the estimation of the means of $T$ distributions (tasks) \citep{Stein:56}. Stein showed that it is better (in a summed squared error sense) to estimate each of the means of $T$ Gaussian random variables using data sampled from all of them, even if the random variables are independent and have different means.  That is,  it is beneficial to consider samples from seemingly \emph{unrelated} distributions in the estimation of the $t$th mean.  This surprising result is often referred to as \emph{Stein's paradox} \citep{EfronMorris:77}.

Mean estimation is perhaps the most common of all estimation tasks, and often multiple means need to be estimated.  In this work we consider a multi-task regularization approach to the problem of estimating multiple means, which we call \emph{multi-task averaging} (MTA).  We show that MTA has provably nice theoretical properties, is effective in practice, and is computationally efficient.

We define the MTA objective in Section \ref{sec:mta}, and review related work in Section \ref{sec:prior_work}.  We present some key properties of MTA in Section \ref{sec:theory}; in particular, we derive the optimal amount of regularization to be used, and show that this optimal amount can be effectively estimated.  Simulations in Section \ref{sec:simulations} verify the advantage of MTA over standard sample means and James-Stein estimation  if the true means are close compared to the variance.  Two applications, (i) estimating expected sales and (ii) estimating final class grades, show that MTA can reduce real errors by over 30\%, as reported in Sections \ref{sec:product_sales} and \ref{sec:grades}.  MTA can be used anywhere multiple averages are needed; we demonstrate this by applying it fruitfully to the averaging in kernel density estimation in Section \ref{sec:KDE}.

\section{Multi-Task Averaging}
\label{sec:mta}

\begin{table}
\caption{Key Notation}
\label{tab:notation}
\begin{center}
\begin{tabular}{l|l}
\hline
&\\
$T$ & number of tasks\\
$N_t$ & number of samples for $t$th task \\
$Y_{ti} \in \mathbb{R}$ & $i$th random sample from $t$th task\\
$\bar{Y}_t \in \mathbb{R}$ & sample average for $t$th task: $\frac{1}{N_t} \sum_i Y_{ti}$\\
$\bar{Y} \in \mathbb{R}^{T}$ & vector with $t$th component $\bar{Y}_t$\\
$Y^*_t \in \mathbb{R}$ & MTA estimate of $t$th  mean\\
$Y^* \in \mathbb{R}^{T}$ & vector with $t$th component $Y^*_t$\\
$\sigma_t^2$ & variance of the $t$th mean\\
$\Sigma$ & diagonal covariance matrix of $\bar{Y}$ with $\Sigma_{tt} = \frac{\sigma_t^2}{N_t}$\\
$A \in \mathbb{R}^{T \times T}$ & pairwise task similarity matrix\\
$L = D - A$ & graph Laplacian of $A$, with diagonal $D$ s.t. $D_{tt} = \sum_{r=1}^T A_{tr}$\\
$W$ & MTA solution matrix, $W = (I + \frac{\gamma}{T}\Sigma L)^{-1}$
\end{tabular}
\end{center}
\end{table}

Consider the problem of estimating the means of $T$ random variables that have finite mean and variance, which is a $T$-task problem from a multi-task learning perspective.  Let  $\{Y_{ti} \}_{i=1}^{N_t}$ be $N_t$ independent and identically distributed (iid) random samples for task $t = 1, \ldots, T$.  Other key notation is in Table \ref{tab:notation}. Assume that the $T \times T$ matrix $A$ describes the relatedness or similarity of any pair of the $T$ tasks, with $A_{tt} = 0$ for all $t$ without loss of generality (because the diagonal self-similarity terms are canceled in the objective below).

The proposed MTA objective is
\begin{align}\label{eqn:mta}
\{Y^*_t\}_{t=1}^T = \argmin_{\{\hat{Y}_t\}_{t=1}^T} ~\frac{1}{T}\sum_{t=1}^T  \sum_{i=1}^{N_t} \frac{(Y_{ti}-\hat{Y}_t)^2}{\sigma_t^2}
+ \frac{\gamma}{T^2} \sum_{r=1}^T\sum_{s=1}^T A_{rs} (\hat{Y}_r-\hat{Y}_s)^2.
\end{align}
The first term of (\ref{eqn:mta}) minimizes the multi-task empirical loss, and the second term jointly regularizes the estimates (i.e. ties them together).  The regularization parameter $\gamma$ balances the empirical risk and the multi-task regularizer. Note that if $\gamma = 0$, the MTA objective decomposes into $T$ separate minimization problems, producing the sample averages $\bar{Y}_t$.

The normalization of each error term in (\ref{eqn:mta}) by its task-specific variance $\sigma_t^2$ (which may be estimated) scales the $T$ empirical loss terms relative to the variance of their distribution; this ensures that high-variance tasks do not disproportionately dominate the loss term.

A more general formulation of MTA is
\begin{align*}
\{Y^*_t\}_{t=1}^T = \argmin_{\{\hat{Y}_t\}_{t=1}^T} ~\frac{1}{T}\sum_{t=1}^T \sum_{i=1}^{N_t} L(Y_{ti}, \hat{Y}_t)
+ \gamma J\left(\{\hat{Y}_t\}_{t=1}^T \right),
\end{align*}
where $L$ is some loss function and $J$ is a regularization function. If $L$ is chosen to be any Bregman loss, then setting $\gamma = 0$ will produce the $T$ sample averages \citep{Banerjee05}. For the analysis and experiments in this paper, we restrict our focus to the tractable squared-error formulation given in (\ref{eqn:mta}). The MTA objective and many of the results in this paper generalize trivially to samples that are vectors rather than scalars, but for notational simplicity we restrict our focus to scalar samples $Y_{ti} \in \mathbb{R}$.

The task similarity matrix $A$ can be specified as side information (e.g. from a domain expert), but often this side information is not available, or it may not be clear how to convert semantic notions of task similarity into an appropriate choice for the task-similarity values in $A$.  In Section \ref{sec:theory}, we derive two optimal choices of $A$ for the $T=2$ case: the $A$ that minimizes expected squared error, and a minimax $A$. We use the $T=2$ analysis to propose  practical estimators of $A$ for any number of tasks.

\section{Related Work}\label{sec:prior_work}
In this section, we review related and background material: James-Stein estimation, multi-task learning, manifold regularization, and the graph Laplacian.

\subsection{James-Stein Estimation}
A closely related body of work to MTA is Stein estimation, an empirical Bayes strategy for estimating multiple means simultaneously \citep{JamesStein:65,Bock72,EfronMorris:77,Casella:85}.  \citet{JamesStein:65} showed that the maximum likelihood estimate of  $\mu_t$  can be dominated by a shrinkage estimate given Gaussian assumptions. Specifically, given a single sample drawn from $T$ normal distributions $Y_t \sim \mathcal{N}(\mu_t,\sigma^2)$ for $t = 1,\ldots,T$,
Stein showed that the maximum likelihood estimator $\bar{Y}_t = Y_t$ is inadmissible, and is dominated by the James-Stein estimator:
\begin{align}\label{eqn:js}
\hat{Y}_t^{JS} = \left(1-\frac{(T-2)\sigma^2}{\bar{Y}^T\bar{Y}}\right)\bar{Y}_t,
\end{align}
where $\bar{Y}$ is a vector with $t$th entry $\bar{Y}_t$.  The above estimator dominates $\bar{Y}_t$ when $T>2$.  For $T = 2$, (\ref{eqn:js}) reverts to the maximum likelihood estimator, which turns out to be admissible \citep{Stein:56}.  James and Stein \citep{JamesStein:65, Casella:85} showed that if $\sigma^2$ is unknown it can be replaced by a standard unbiased estimate $\hat{\sigma}^2$.

Note that in (\ref{eqn:js}) the James-Stein estimator \emph{shrinks} the maximum likelihood estimates towards zero (the terms ``regularization'' and ``shrinkage'' are often used interchangeably).  The form of (\ref{eqn:js}) and its shrinkage towards zero points to the implicit assumption that the $\mu_t$ are drawn from a standard normal distribution centered at 0.  More generally, the means are assumed to be drawn as $\mu_t \sim \mathcal{N}(\xi,1)$.  The James-Stein estimator then becomes
\begin{align}\label{eqn:jsNotAdmiss}
\hat{Y}_t^{JS} = \xi + \left(1-\frac{(T-3)\sigma^2}{(Y-\xi)^T(Y-\xi)}\right)(\bar{Y}_t - \xi),
\end{align}
where $\xi$ can be estimated (as we do in this work) as the average of means $\xi = \bar{\bar{Y}} = \frac{1}{T}\sum_{r=1}^T \bar{Y}_r$, and this additional estimation decreases the degrees of freedom by one\footnote{For more details as to why $T-2$ in (\ref{eqn:js}) becomes $T-3$ in (\ref{eqn:jsNotAdmiss}) see Example 7.7 on page 278 of \citet{LehmannCasella}.}.

There have been a number of extensions to the original James-Stein estimator.   Throughout this work, we compare to the well-regarded positive-part James-Stein estimator for multiple data points per task and independent unequal variances \citep{Bock72,LehmannCasella}.  In particular, let $Y_{ti} \sim \mathcal{N}(\mu_t,\sigma_t^2)$ for $t = 1,\ldots,T$ and $i = 1,\ldots,N_t$, let $\Sigma$ be the covariance matrix of $\bar{Y}$, the vector of task sample means, and let $\lambda_{\text{max}}(\Sigma)$ be the largest eigenvalue of $\Sigma$.  The James-Stein estimator given in (\ref{eqn:jsNotAdmiss}) is itself is not admissible, and is dominated by the positive part James-Stein estimator \citep{LehmannCasella}, which is further theoretically improved by Bock's James-Stein estimator \citep{Bock72}:
\begin{equation} \label{eqn:jsBock}
\hat{Y}_t^{JS} = \xi + \left(1-\frac{\frac{\tr(\Sigma)}{\lambda_{\text{max}}(\Sigma)} - 3}{(\bar{Y}- \xi )^T\Sigma^{-1}(\bar{Y}- \xi )}\right)_+ (\bar{Y}_t - \xi),
\end{equation}
where  $(x)_+ = \max(0, x)$. The term $\frac{\tr(\Sigma)}{\lambda_{\text{max}}}$ is called the \emph{effective dimension} of the estimator. In simulations where we set the true covariance matrix to be $\Sigma$ and then estimated the effective dimension by estimating the maximum eigenvalue and trace of the sample covariance matrix, we found that replacing the effective dimension with the actual dimension $T$ (when $\Sigma$ is diagonal) resulted in a significant performance boost for Bock's James-Stein estimator.  For the case of a diagonal $\Sigma$, there are $T$ separate distributions, thus the effective dimension is exactly $T$. In other preliminary experiments with real data, we also found that using $T$ rather than the effective dimension performed better due to the high variance of the estimated maximum eigenvalue in the denominator of the effective dimension.  Consequently, in the experiments in this paper, when we compare to James-Stein estimation, we compare to (\ref{eqn:jsBock}) using $T$ for the effective dimension.


\subsection{Multi-Task Learning for Mean Estimation}
MTA is an approach to the problem of estimating $T$ means. We are not aware of other work in the multi-task literature that addresses this problem explicitly; most MTL methods are designed for regression, classification, or feature selection, e.g. \citet{Micchelli:04,Bonilla:08,Argyriou08}.  Estimating $T$ means can be considered a special case of multi-task regression\footnote{With a feature space of zero dimensions, only the constant offset term is learned.}, where one fits a constant function to each task.  And, similarly to MTA, one of the main approaches to multi-task regression in literature is tying tasks together with an explicit multi-task parameter regularizer.

\citet{Abernethy09}, for instance, propose to minimize the empirical loss with the following added regularizer,
$$ ||\beta||_*,$$
where the $t$th column of the matrix $\beta$ is the vector of parameters for the $t$th task and $||\cdot||_*$ is the trace norm.  For mean estimation, the matrix $\beta$ has only one row, and its trace norm has little meaning.

\citet{Argyriou08} propose an alternating approach with a different regularizer,
$$\tr(\beta^TD^{-1}\beta),$$
where $D$ is a learned, shared \emph{feature} covariance matrix.  Again, with no features, $D$ is just a constant.  The regularizers in the work of \citet{Jacob08} and \citet{ZhangUAI10} are similarly inappropriate when in the context of mean estimation.

The most closely related work is that of \citet{SheldonNIPS08} and \citet{KatoNIPS07}, where the regularizer or constraint, respectively, is
$$\sum_{r=1}^T\sum_{s=1}^T A_{rs}\|\beta_r-\beta_s\|_2^2,$$
which is the MTA regularizer when performing mean estimation.

\subsection{Multi-Task Learning and the Similarity Between Tasks}
A key issue for MTA and many other multi-task learning methods is how to estimate some notion of similarity (or task relatedness) between tasks and/or samples if it is not provided. A common approach is to estimate the similarity matrix jointly with the task parameters \citep{Argyriou07,Xue07,Bonilla:08,Jacob08,ZhangUAI10}. For example, \citet{ZhangUAI10} assume that there exists a covariance matrix for the task relatedness, and proposed a convex optimization approach to estimate the task covariance matrix and the task parameters in a joint, alternating way.  Applying such joint and alternating approaches to the MTA objective given in (\ref{eqn:mta}) leads to a degenerate solution with zero similarity.  However, the simplicity of MTA enables us to specify the optimal task similarity matrix for $T=2$ (see Sec. \ref{sec:theory}), which we use to obtain a number of closed-form estimators for the general $T > 1$ case.

\subsection{Manifold Regularization}
MTA is similar in form to \emph{manifold regularization} \citep{Belkin:06}.
For example, Belkin et al.'s Laplacian-regularized least squares objective for semi-supervised regression solves
\begin{eqnarray*}
\argmin_{f \in \H} & \sum_{i=1}^{N}(y_i - f(x_i))^2 +
 \lambda ||f||_{\H}^2 + \gamma \sum_{i,j=1}^{N+M} A_{ij}(f(x_i)-f(x_j))^2,
\end{eqnarray*}
where $f$ is the regression function to be estimated, $\H$ is a reproducing kernel Hilbert space (RKHS),
$N$ is the number of labeled training samples, $M$ is the number
of unlabeled training samples, $A_{ij}$ is the similarity
(or weight in an adjacency graph) between feature samples $x_i$ and $x_j$,
and $||f||_{\H}$ is the norm of the function $f$ in the RKHS.
In MTA, as opposed to manifold regularization, we are estimating a different function (that is, the constant function that is the mean) for each of the $T$ tasks, rather than a single global function. One can interpret MTA as regularizing the individual task estimates over the task-similarity manifold, which is defined for the $T$ tasks by the $T \times T$ matrix $A$.

\subsection{Background on the Graph Laplacian Matrix}\label{sec:laplacian}
It will be helpful for later sections to review the graph Laplacian matrix. For graph $G$ with $T$ nodes, let $A \in \R^{T\times T}$ be a matrix where component $A_{rs} \geq 0$ is the weight of the edge between node $r$ and node $s$, for all $r,s$.  The \emph{graph Laplacian matrix} is defined as $L = L(A) = D - A$, with diagonal matrix $D$ such that $D_{tt} = \sum_{s} A_{ts}$.

The graph Laplacian  matrix is analogous to the Laplacian operator $\Delta g(x) = \tr(H(g(x))) = \frac{\partial^2g(x)}{\partial x_1^2} + \frac{\partial^2g(x)}{\partial x_2^2} + \ldots + \frac{\partial^2g(x)}{\partial x_M^2}$, which quantifies how locally smooth a twice-differentiable function $g(x)$ is.  Similarly, the graph Laplacian matrix $L$  can be thought of as being a measure of the smoothness of a function defined on a graph \citep{Chung94}.  Given a function $f$ defined over the $T$ nodes of graph $G$, where $f_i \in \R$ is the function value at node $i$, the total \emph{energy} of a graph is (for symmetric $A$)
$$\mathcal{E}(f) = \frac{1}{2}\sum_{i=1}^T\sum_{j=1}^T A_{ij}(f_i - f_j)^2 = f^TL(A)f,$$
which is small when $f$ is smooth over the graph \citep{Zhu05}.  If $A$ is asymmetric then the energy can be written as
$$\mathcal{E}(f) = \frac{1}{2}\sum_{i=1}^T\sum_{j=1}^T A_{ij}(f_i - f_j)^2 = f^TL((A+A^T)/2)f.$$
Note that the above formulation of the energy in terms of the graph Laplacian holds for the scalar case.  More generally, when each $f_i \in \R^d$ is a vector, one can alternatively write the energy in terms of the distance matrix:
$$\mathcal{E}(f) = \frac{1}{2}\tr(\Delta^T A),$$
where $\Delta_{ij} = (f_i-f_j)^T(f_i-f_j)$

As discussed above, the graph Laplacian can be thought of as an operator on a function, but it is useful in and of itself (i.e. without a function).  The eigenvalues of the graph Laplacian are all real and non-negative, and there is a wealth of literature showing how the eigenvalues reveal the structure of the underlying graph \citep{Chung94}; the eigenvalues of $L$ are particularly useful for spectral clustering \citep{Luxburg07}.  The graph Laplacian is a common tool in semi-supervised learning literature \citep{Zhu06}, and the Laplacian of a random walk probability matrix $P$ (i.e. all the entries are non-negative and the rows sum to 1) is also of interest. For example, Saerens et al. \citep{Saerens04} showed that the pseudo-inverse of the Laplacian of a probability transition matrix is used to compute the square root of the average commute time (the average time taken by a random walker on graph $G$ to reach node $j$ for the first time when starting at node $i$, and coming back to node $i$). 


\section{MTA Theory}
\label{sec:theory}
We derive a closed-form solution for $A$ and various properties. Proofs and derivations are in the appendix.

\subsection{Closed-form MTA Solution}\label{sec:solution}
For symmetric $A$ with non-negative components\footnote{Using an asymmetric $A$  with MTA is equivalent to using the symmetric matrix $(A^T + A)/2$.}, the MTA objective given in (\ref{eqn:mta}) is continuous, differentiable, and convex; and (\ref{eqn:mta}) has closed-form solution:
\begin{align}\label{eqn:mta_sln}
Y^* &= \left(I + \frac{\gamma}{T} \Sigma L\right)^{-1}\bar{Y}
\end{align}
where $\bar{Y}$ is the vector of sample averages with $t$th entry $\bar{Y}_t = \frac{1}{N_t}\sum_{i=1}^{N_t} Y_{ti}$, $L$ is the graph Laplacian of $A$, and $\Sigma$ is the diagonal covariance matrix of  the sample mean vector $\bar{Y}$ such that $\Sigma_{tt} = \frac{\sigma_t^2}{N_t}$. The inverse $\left(I + \frac{\gamma}{T} \Sigma L\right)^{-1}$ always exists:

\begin{lemma}Assume that $0 \leq A_{rs} < \infty $ for all $r,s$, $\gamma \geq 0$, and $0 < \frac{\sigma_t^2}{N_t} < \infty$ for all $t$.  The MTA solution matrix $W = \left(I + \frac{\gamma}{T} \Sigma L\right)^{-1}$ exists.
\end{lemma}

Note that the $(r,s)$th entry of $\frac{\gamma}{T} \Sigma L$ goes to $0$ as $N_t$ approaches infinity, and since matrix
inversion is a continuous operation, $\left(I + \frac{\gamma}{T} \Sigma L\right)^{-1} \to I$ in the norm. By the law of large numbers one can conclude that $Y^*$ asymptotically
approaches the true mean $\mu$.

MTA can also be applied to vectors.  Let $\mathbb{\bar{Y}}^* \in \R^{T\times d}$ be a matrix with $Y^*_t$ as its $t$th row and let $\mathbb{\bar{Y}} \in \R^{T\times d}$ be a matrix with $\bar{Y}_t \in \R^d$ as its $t$th row.   One can simply perform MTA on the vectorized form of $\mathbb{Y}^*$.

\subsection{Regularized Laplacian Kernel}
The MTA solution matrix $W = \left(I + \frac{\gamma}{T} \Sigma L\right)^{-1}$ is similar to the \emph{regularized Laplacian kernel} (RLK): $Q = (I + \gamma L)^{-1}$, introduced by Smola and Kondor \citep{Smola03}.  In the RLK, the graph Laplacian matrix $L$ is assumed to be symmetric, but the $\Sigma L$ in the MTA solution matrix is generally not symmetric.  The MTA solution matrix therefore generalizes the RLK.

Note that the term \emph{kernel} refers to a positive semi-definite matrix used in, for example, support vector machines \citep{HTF}.  The $(r,s)$th entry of any kernel matrix can be interpreted as a similarity between the $r$th and $s$th samples.  In this section, we will discuss and motivate the kind of similarity that is encoded by both the RLK and the MTA solution matrix $W$.

Chebotarev and Shamis \citep{Chebotarev06} studied matrices of the form $Q = (I + \gamma L)^{-1}$ in the context of answering the question ``given a graph, how should one evaluate the proximity between its vertices?''  They prove a number of properties that lead them to conclude that $Q_{ij}$ is a good measure of how \emph{accessible} $j$ is from $i$ when taking all possible paths into account (as opposed to just the direct path that $A_{ij}$ encodes).  In their own words, ``$Q_{ij}$ may be interpreted as the fraction of the connectivity of vertices $i$ and $j$ in the total connectivity of $i$ with all vertices.'' The following is a list of interesting properties of $Q$ from the work of Chebotarev and Shamis  when $A$ is symmetric and its entries are ``strictly positive'' \citep{Chebotarev06}:
\begin{itemize}
  \item $Q$ exists and has convex rows.
  \item $Q_{ii} > Q_{ij}$.
  \item Triangle inequality: $Q_{ij} + Q_{ik} - Q_{jk} \leq Q_{ii}$.
  \item The distance $d^\alpha_{ij} = \alpha(Q_{ii}^\alpha + Q_{jj}^\alpha - Q_{ij}^\alpha - Q_{ji}^\alpha)$ is a valid metric distance over vertices.
  \item $Q_{ij} = 0$ if and only if there exists no path between $i$ and $j$.
\end{itemize}
For intuition as to why $Q$ measures connectivity, consider the following expansion \citep{Berman79}:
$$(I+\gamma L)^{-1} = \sum_{k=0}^\infty (-\gamma L)^k.$$
This equality holds only if the right-hand side is convergent.  Thus, the RLK is a type of path counting with $-L$ instead of $A$ as the adjacency matrix, where paths of all possible lengths are taken into account, and longer paths are weighted: equally ($\gamma = 1$), less heavily ($\gamma < 1$), or more heavily ($\gamma > 1$).

The MTA solution matrix $\left(I + \frac{\gamma}{T} \Sigma L\right)^{-1}$ generalizes the RLK; the diagonal matrix $\Sigma$ left-multiplies the Laplacian, and the RLK is produced in the special case that $\Sigma = cI$ for any scalar $c$.  Using a different approach than Chebotarev and Shamis, we will prove in the next subsection that the convexity of the rows of $W$ still holds, assuming only non-negativity of the entries of $A$ (instead of strict positivity as in Chebotarev and Shamis).  (We did not investigate whether the other properties listed above still hold for the MTA solution.)

The RLK is one of many possible graph kernels.  To find the best one for a collaborative recommendation task, Fouss et al. \citep{Fouss06} empirically compared seven graph kernels.  They found that the best three kernels were the RLK, the pseudo-inverse of $L$, and the Markov diffusion kernel.  Yajima and Kuo \citep{Yajima06} tested various graph kernels in the context of a one-class SVM for the application of recommendation tasks.  They also found that the RLK was one of the top performers.

\subsection{Convexity of MTA Solution}
From inspection of (\ref{eqn:mta_sln}), it is clear that each 
of the elements of the MTA solution 
$Y^*$ is a linear combination of the single-task sample averages in $\bar{Y}$.  In fact, each MTA estimate is a convex combination of the single-task sample averages: \\

\noindent \textbf{Theorem} \emph{If  $\gamma \geq 0$, $0 \leq A_{rs} < \infty $ for all $r,s$ and $0 < \frac{\sigma_t^2}{N_t} < \infty$ for all $t$, then the MTA estimates $\{Y^*_t\}$ given in (\ref{eqn:mta_sln}) are a convex combination of the task sample averages $\{\bar{Y}_t\}$.}

\subsection{Analysis of the Two Task Case}\label{sec:optimal_a}
In this section we analyze the $T=2$ task case, with $N_1$ and $N_2$ samples for tasks 1 and 2 respectively.  Suppose  $\{Y_{1i}\}$ are iid  with finite mean $\mu_1$ and finite variance $\sigma_1^2$, and $\{Y_{2i}\}$ are iid with finite mean $\mu_2 = \mu_1 + \Delta$ and finite variance $\sigma_2^2$. Let the task-relatedness matrix be $A = [0 ~ a; a ~ 0]$, and without loss of generality, we fix  $\gamma = 1$.  Then the closed-form solution (\ref{eqn:mta_sln}) can be simplified:
\begin{equation}\label{eqn:two_task}
Y_1^* = \left( \frac{ T + \frac{\sigma_2^2}{N_2}a }{ T + \frac{\sigma_1^2}{N_1}a + \frac{\sigma_2^2}{N_2}a  } \right) \bar{Y}_1
       + \left( \frac{\frac{\sigma_1^2}{N_1}a }{T + \frac{\sigma_1^2}{N_1}a  + \frac{\sigma_2^2}{N_2}a  } \right) \bar{Y}_2.
\end{equation}
The mean squared error of $Y_1^*$ is
\begin{align*}
\MSE[Y_1^*] = \frac{\sigma_1^2}{N_1}
\left( \frac{ T^2 + 2T\frac{\sigma_2^2}{N_2}a +   \frac{\sigma_1^2\sigma_2^2}{N_1N_2}a^2 +  \frac{\sigma_2^4}{N^2_2}a^2}
     {(T + \frac{\sigma_1^2}{N_1}a  + \frac{\sigma_2^2}{N_2}a )^2  } \right)
     +   \frac{\Delta^2 \frac{\sigma_1^4}{N_1^2}a^2 }{(T + \frac{\sigma_1^2}{N_1}a  + \frac{\sigma_2^2}{N_2}a )^2  }.
\end{align*}
Next, we compare the MTA estimate to the sample average $\bar{Y}_1$, which is the maximum likelihood estimate of the mean for many distributions.\footnote{The uniform distribution is perhaps the simplest example where the sample average is not the maximum likelihood estimate of the mean. For more examples, see Sec. 8.18 of \citet{Counterexamples}.} The MSE of the single-task sample average $\bar{Y}_1$ is $\frac{\sigma_1^2}{N_1}$, and thus
\begin{align}\label{eqn:mseLower}
\MSE[Y_1^*]  < \MSE[\bar{Y}_1] \textrm{ if } \Delta^2  <  \frac{4}{a} + \frac{\sigma_1^2}{N_1} + \frac{\sigma_2^2}{N_2},
\end{align}
Thus the MTA estimate of the first mean has lower MSE than the sample average estimate if the squared mean-separation $\Delta^2$ is small compared to the summed variances of the sample means.  See Figure \ref{fig:twoTaskCase} for an illustration.

Note that as $a$ approaches 0 from above, the term $4/a$ in (\ref{eqn:mseLower}) approaches infinity, which means that a small amount of regularization can be helpful even when the difference between the task means $\Delta$ is large.

\begin{figure}[t!]
\centering
\includegraphics[width = 6.0in]{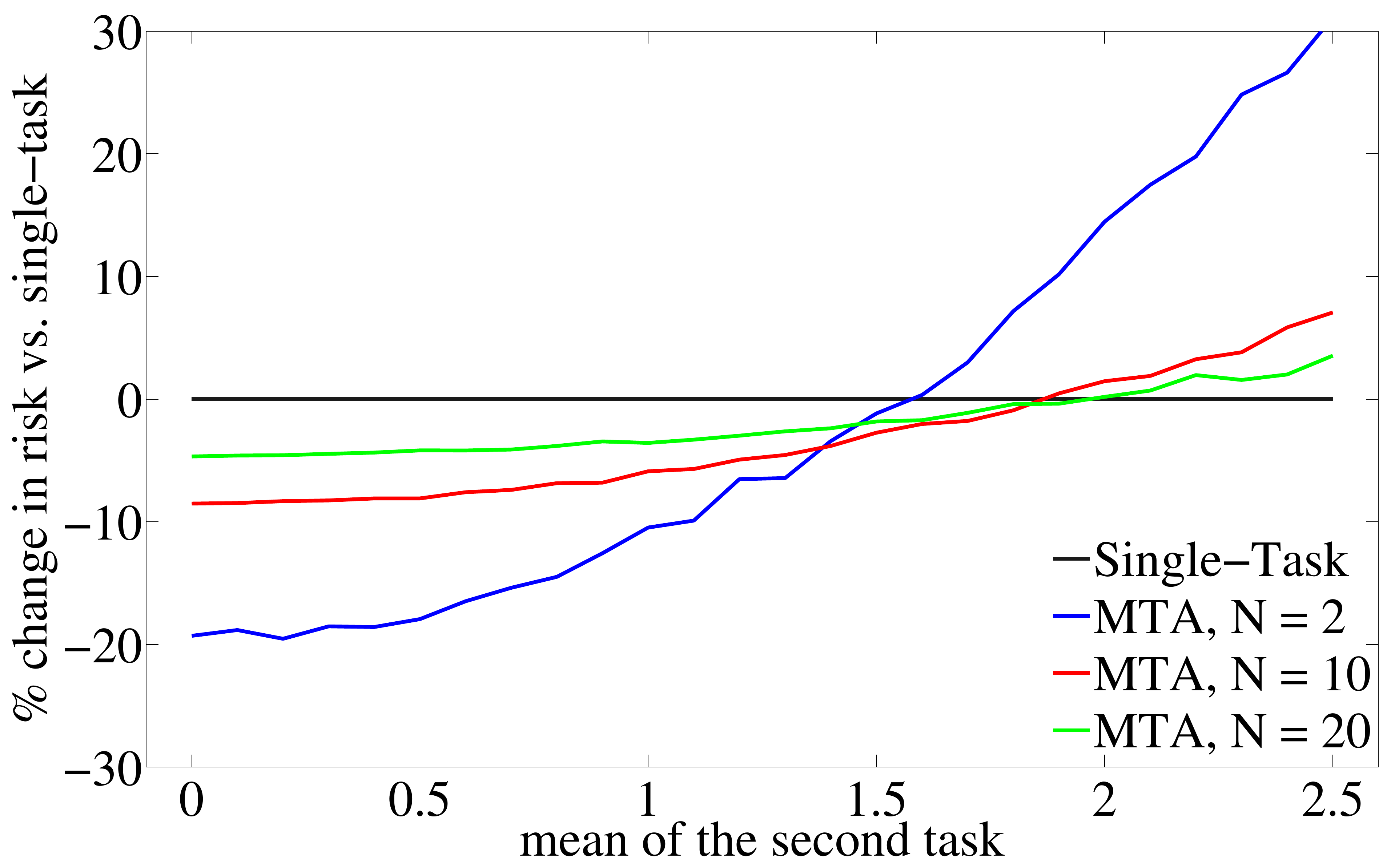} \\
\caption{\footnotesize Plot shows the percent change in average risk for two tasks (averaged over 10,000 runs of the simulation). For each task there are $N$ IID samples, for $N=2,10,20$.  The first task generates samples from a standard Gaussian. The second task generates samples from a Gaussian with $\sigma^2 = 1$ and varying mean value, as marked on the x-axis. The symmetric task-relatedness value was fixed at $a = 1$ (note this is generally not the optimal value). One sees that given $N=2$ samples from each Gaussian, the MTA estimate is better if the Gaussians are closer than 2 units apart. Given $N=20$ samples from each Gaussian, the MTA estimate is better if the Gaussians are closer than 1.5 units apart. In the extreme case that the two Gaussians have the same mean ($\mu_1=\mu_2 =0$), then with this suboptimal choice of $a=1$, MTA provides a $20\%$ win for $N=2$ samples,  and a 5$\%$ win for $N=20$ samples.} \label{fig:twoTaskCase}
\end{figure}

\subsection{Optimal Task Relatedness $A$ for $T=2$}
We analyze the optimal choice of $a$ in the task-similarity matrix $A = [0 ~ a; a ~ 0]$. The risk is the sum of the mean squared errors: $$R(\mu,Y^*) = \MSE[Y_1^*] + \MSE[Y_2^*],$$ which is a convex, continuous, and differentiable function of $a$, and therefore the first derivative can be used to specify the optimal value $a^*$, when all the other variables are fixed.  Minimizing the risk $\MSE[Y_1^*] + \MSE[Y_2^*]$ w.r.t. $a$ one obtains the following solution:
\begin{equation}\label{eqn:optimal_a}
a^*  = \frac{2}{\Delta^2},
\end{equation}
which is always non-negative, as was assumed.   This result is key because it specifies that the optimal task-similarity $a^*$ ideally should measure the inverse of the squared task mean-difference.  Further, the optimal task-similarity is independent of the number of samples $N_t$ or the sample variance $\sigma_t^2$, as these are accounted for in $\Sigma$.   Note that $a^*$ also minimizes the functions $\MSE[Y_1^*]$ and $\MSE[Y_2^*]$, separately.

Analysis of the second derivative shows that this minimizer always holds for the cases of interest (that is, for $N_1, N_2 \geq 1$).
The effect on the risk of the choice of $a$ and the optimal $a^*$ is illustrated in Figure \ref{fig:OptimalRelatedness}.
\begin{figure}[t!]
\centering
\includegraphics[width = 6.0in]{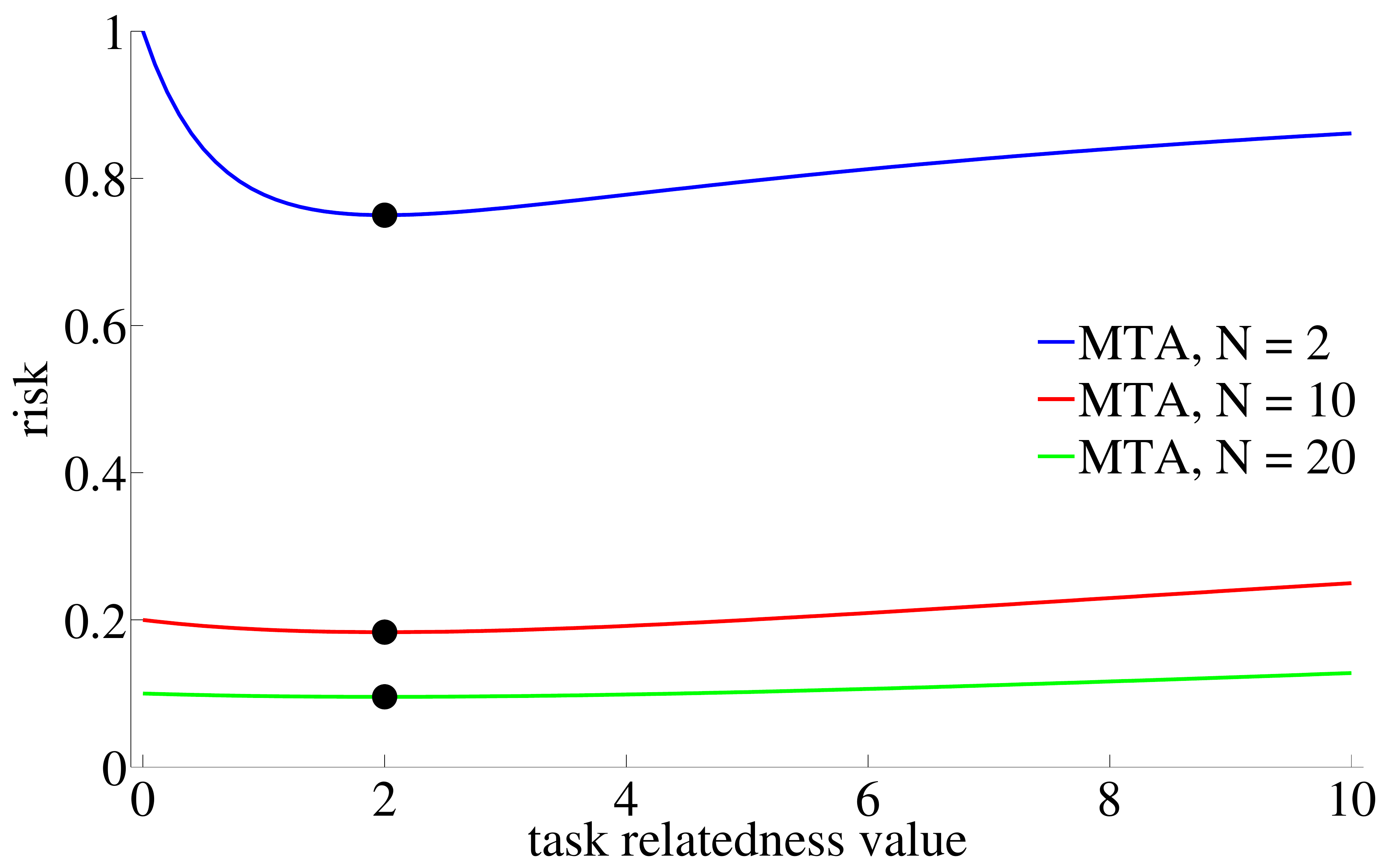} \\
\caption{\footnotesize Plot shows the risk for two tasks, where the task samples were drawn IID from Gaussians $\mathcal{N}(0,1)$ and $\mathcal{N}(1,1)$. The task-relatedness value $a$ was varied as shown on the x-axis. The minimum expected squared error is marked by a $*$, is independent of $N$ and matches the optimal task-relatedness value given by (\ref{eqn:optimal_a}).} \label{fig:OptimalRelatedness}
\end{figure}
The optimal two-task similarity given in (\ref{eqn:optimal_a}) requires knowledge of the true means $\mu_1$ and $\mu_2$.  These are, in practice, unavailable.  What similarity should be used then?  A straightforward approach is to use single-task estimates instead:
\begin{equation*}
\hat{a}^* = \frac{2}{(\bar{y}_1-\bar{y}_2)^2},
\end{equation*}
And to use maximum likelihood estimates $\hat{\sigma}_t^2$ to form the matrix $\hat{\Sigma}$.  This data-dependent approach is analogous to empirical Bayesian methods in which prior parameters are estimated from data \citep{Casella:85}.

\subsection{Estimating $A$ from Data for Arbitrary $T$}\label{sec:estimation}
Based on our analysis in the preceding sections of the optimal $A$ for the two-task case, we propose two methods to estimate $A$ from data for arbitrary $T>1$. The first method is designed to minimize the approximate risk using a constant similarity matrix. The second method provides a minimax estimator.  With both methods one can take advantage of the Sherman-Morrison formula \citep{ShermanMorrison} to avoid taking the matrix inverse or solving a set of linear equations in (\ref{eqn:mta_sln}), resulting in an $O(T)$ computation time for $Y^*$.

\subsubsection{Constant MTA}
Recalling that $E[\bar{Y}\bar{Y}^T] = \mu\mu^T + \Sigma$, the \emph{risk} of estimator $\hat{Y} = W\bar{Y}$ of unknown parameter vector $\mu$ for the squared loss is the sum of the mean squared errors:
\begin{align}\label{eqn:risk}
R(\mu,W\bar{Y}) &= E[(W\bar{Y} - \mu)^T(W\bar{Y} - \mu)] \nonumber\\
 &= \tr(W\Sigma W^T) + \mu^T(I-W)^T(I-W)\mu.
\end{align}
One approach to generalizing the results of Section \ref{sec:optimal_a} to arbitrary $T$ is to try to find a symmetric, non-negative matrix $A$ such that the (convex, differentiable) risk $R(\mu,W\bar{Y})$ is minimized for $W =\left(I + \frac{\gamma}{T} \Sigma L\right)^{-1}$ (recall $L$ is the graph Laplacian of $A$).  The problem with this approach is two-fold:  (i) the solution is not analytically tractable for $T>2$ and (ii) an arbitrary $A$ has $T(T-1)$ degrees of freedom, which is considerably more than the number of means we are trying to estimate in the first place. To avoid these problems, we generalize the two-task results by constraining $A$ to be a scaled constant matrix $A = a\one\one^T$, and find the optimal $a^*$ that minimizes the risk in (\ref{eqn:risk}).  In addition, w.l.o.g. we set $\gamma$ to 1, and for analytic tractability we assume that all the tasks have the same variance, estimating  $\Sigma$ as $\frac{\tr(\Sigma)}{T}I$. Then it remains to solve:
$$a^* = \argmin_a R\left(\mu, \left(I + \frac{1}{T} \frac{\tr(\Sigma)}{T} L(a\one\one^T)\right)^{-1}\bar{Y}\right),$$
which has the solution
$$a^* = \frac{2}{\frac{1}{T(T-1)}\sum_{r=1}^T\sum_{s=1}^T (\mu_r - \mu_s)^2},$$
which reduces to the optimal two task MTA solution (\ref{eqn:optimal_a}) when $T=2$. In practice, one of course does not have $\{\mu_r\}$ as these are precisely the values one is trying to estimate.  So, to estimate $a^*$ we use the sample means $\{\bar{y}_r\}$:
$$\hat{a}^* = \frac{2}{\frac{1}{T(T-1)}\sum_{r=1}^T\sum_{s=1}^T (\bar{y}_r - \bar{y}_s)^2}.$$

Using this optimal estimated \emph{constant} similarity and an estimated covariance matrix $\hat{\Sigma}$ produces what we refer to as the \emph{constant MTA} estimate
\begin{equation}
Y^* = \left(I + \frac{\gamma}{T}\hat{\Sigma} L(\hat{a}^* \one \one^T)\right)^{-1} \bar{Y}.\label{eqn:constantMTA}
\end{equation}
Note that we made the assumption that the entries of $\Sigma$ were the same in order to be able to compute the constant similarity $a^*$, but we do not need nor suggest that assumption when using $a^*$ in (\ref{eqn:constantMTA}).

To compute this estimate one needs the diagonal matrix $\Sigma$, which in practice also must be estimated.

\subsubsection{Minimax MTA}
Bock's James-Stein estimator is \emph{minimax}, that is, it minimizes the worst-case loss, and not necessarily the expected risk \citep{LehmannCasella}.   This leads to a more conservative use of regularization.  In this section, we derive a  minimax version of MTA for arbitrary $T$ that prescribes less regularization than constant MTA.  Formally, an estimator $Y^M$ of $\mu$ is called minimax if it minimizes the maximum risk:
$$\inf_{\hat{Y}} \sup_{\mu} R(\mu,\hat{Y}) = \sup_\mu R(\mu,Y^M).$$
Let $r(\pi,\hat{Y})$ be the average risk of estimator $\hat{Y}$ w.r.t. a prior $\pi(\mu)$ such that  $r(\pi,\hat{Y}) = \int R(\mu,\hat{Y})\pi(\mu)d\mu$.  The Bayes estimator $Y^\pi$ is the estimator that minimizes the average risk, and the Bayes risk $r(\pi,Y^\pi)$ is the average risk of the Bayes estimator.  A prior distribution $\pi$ is called least favorable if
$r(\pi,Y^\pi) > r(\pi',Y^{\pi'})$ for all priors $\pi'$.

First, we will specify minimax MTA for the $T=2$ case.  To find a minimax estimator $Y^M$ it is sufficient to show that \emph{(i)} $Y^M$ is a Bayes estimator w.r.t. the least favorable prior (LFP) and \emph{(ii)} it has constant risk \citep{LehmannCasella}.  To find a LFP, we first need to specify a constraint set for $\mu_t$: we use an interval: $\mu_t \in [b_l,b_u],$ for all $t$, where $b_l \in \R$ and $b_u \in \R$. With this constraint set the minimax estimator is (see appendix for details):
\begin{equation*}
Y^M = \left(I + \frac{2\gamma}{T(b_u - b_l)^2} \Sigma L(\one\one^T)\right)^{-1}\bar{Y},
\end{equation*}
which reduces to (\ref{eqn:optimal_a}) when $T=2$.  This minimax analysis is only valid for the case when $T=2$, but we found that the following extension of minimax MTA to larger $T$ worked well in simulations and applications for any $T \geq 2$.  To estimate $b_u$ and $b_l$ from data we assume the unknown $T$ means are drawn from a uniform distribution and use maximum likelihood estimates of the lower and upper endpoints for the support: $$\hat{b}_l = \min_t \bar{y}_t ~~ \text{and} ~~ \hat{b}_u = \max_t \bar{y}_t.$$
Thus, in practice, \emph{minimax MTA} is
\begin{equation*}
Y^M = \left(I + \frac{2\gamma}{T(\hat{b}_u - \hat{b}_l)^2} \hat{\Sigma} L(\one\one^T)\right)^{-1}\bar{Y},
\end{equation*}

\subsubsection{Computational Efficiency of Constant and Minimax MTA}
Both the constant MTA and minimax MTA weight matrices can be written as
\begin{align*}
(I + c\Sigma L(\one\one^T))^{-1} &= (I + c\Sigma (TI - \one\one^T ))^{-1}\\
&= (I + cT \Sigma - c\Sigma \one\one^T )^{-1}\\
&= (Z - x\one^T)^{-1},
\end{align*}
where $c$ is different for constant MTA and minimax MTA, $Z = I + cT \Sigma$, $x = c\Sigma\one$.  The matrix $Z$ is diagonal (since $\Sigma$ is diagonal), and thus the Sherman-Morrison formula \citep{ShermanMorrison} can be used to find the inverse:
$$(Z - x\one^T)^{-1} = Z^{-1} + \frac{Z^{-1}x\one^TZ^{-1}}{1+\one^TZ^{-1}x}.$$
Since $Z$ is diagonal, $Z^{-1}$ can be computed in $O(T)$ time, and so can $Z^{-1}x$.  Thus, the entire computation $W\bar{Y}$ can be done in $O(T)$ time for constant MTA and minimax MTA.

\subsection{Generality of Matrices of MTA Form} \label{sec:other_ests}
Figure \ref{fig:sets_of_estimators} is a visual summary of the sets of estimators of type $\hat{Y} = W\bar{Y}$, where $W$ is a $T \times T$ matrix.  The pink region represents estimators of the form $\hat{Y} = W\bar{Y}$, with right-stochastic $W$.  MTA estimators are all within the green region, and many well-known estimators (such as the James-Stein Estimator and its variants) fall within the purple region.  In this section we will prove that the purple region is a strict subset of the the green region with a proposition.  In other words, we will show that MTA generalizes many estimators of interest, such as estimators that regularize single-task estimates of the mean to the pooled mean or the average of means.

\begin{figure}[th!]
\centering
\includegraphics[width = 6in]{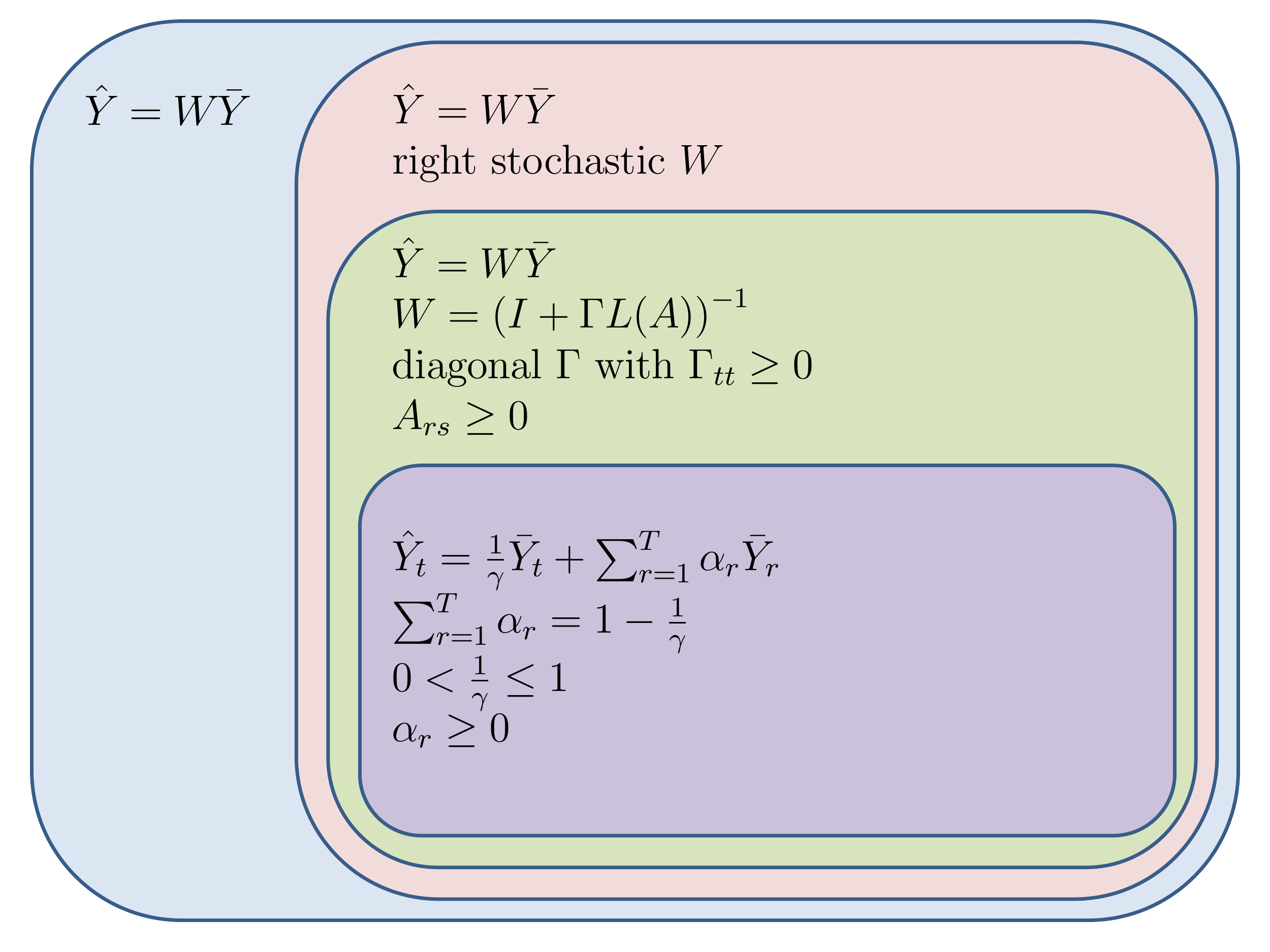}\\
\caption{An illustration of the set membership properties of various estimators of the type $\hat{Y} = W\bar{Y}$.} \label{fig:sets_of_estimators}
\end{figure}

Specifically, the proposition will establish that familiar regularized estimates of $\mu$ can be rewritten in MTA form for specific choices of (or assumptions about) $A$, $\gamma$, and $\Sigma$.  Note that the covariance $\Sigma$ is also a ``choice" because some classic estimators assume $\Sigma = I$.  First, recall the MTA solution:
\begin{equation*}
Y^* = \left(I + \frac{\gamma}{T} \Sigma L\right)^{-1}\bar{Y}.
\end{equation*}
In the following sections we refer to matrices of \emph{MTA form}.  In the most general case, this form is
\begin{equation}\label{eq:mta_general}
\left(I + \Gamma L(A) \right)^{-1},
\end{equation}
where $A$ is a matrix with all non-negative entries, and $\Gamma$ is a diagonal matrix with all non-negative entries.

\begin{proposition}
The set of estimators $W\bar{Y}$ where $W$ is of MTA form as per (\ref{eq:mta_general}) is strictly larger than the set of estimators that regularize the single-task estimates as follows:
\begin{align*}
\hat{Y}_t = \frac{1}{\gamma} \bar{y}_t + \sum_{r=1}^T \alpha_r \bar{Y}_r,
\end{align*}
where $\sum_{r=1}^T \alpha_r = 1-\frac{1}{\gamma}$, $0 < \frac{1}{\gamma} \leq 1$, and $\alpha_r \geq 0$, $\forall r$.
\end{proposition}

\begin{corollary}
Estimators which regularize the single task estimate towards the pooled mean such that they can be written
\begin{align*}
\check{Y}_t = \lambda \bar{Y}_t + \frac{1-\lambda}{\sum_{r=1}^T N_r} \sum_{s=1}^T\sum_{i=1}^{N_s} Y_{si},
\end{align*}
for $\lambda \in (0,1]$ can also be written in MTA form as
$$\check{Y} = \left(I + \frac{1-\lambda}{\lambda \mathbf{N}^T\one} L(\one \mathbf{N}^T)\right)^{-1}\bar{Y},$$
where $\mathbf{N}$ is a $T$ by $1$ vector with $N_t$ as its $t$th entry, with corresponding choices if $A$ and $\Gamma$ obtained by visual pattern matching to (\ref{eq:mta_general}).
\end{corollary}

\begin{corollary}
Estimators which regularize the single task estimate towards the average of means (AM) such that they can be written
\begin{align*}
\breve{Y}_t = \lambda \bar{Y}_t + \frac{1-\lambda}{T} \sum_{t=1}^T \bar{Y}_t,
\end{align*}
for $\lambda \in (0,1]$, can also be written in MTA form as
\begin{align*}
\breve{Y} = \left(I+\frac{1-\lambda}{\lambda T}L(\one\one^T)\right)^{-1}\bar{Y},
\end{align*}
with corresponding choices if $A$ and $\Gamma$ obtained by visual matching to (\ref{eq:mta_general}).
\end{corollary}

Note that the proof of proposition in the appendix uses MTA form with \emph{asymmetric} similarity matrix $A = \one\alpha^T$.  And, indeed, there is nothing about the MTA solution that requires $A$ to be symmetric.
Initially, we constrained $A$ to be symmetric because of the form of the regularizer in the objective (\ref{eqn:mta}):
$$\frac{1}{2}\sum_{r=1}^T\sum_{s=1}^T A_{rs}(\hat{Y}_r - \hat{Y}_s)^2 = \hat{Y}^T L((A+A^T)/2)\hat{Y}^T.$$
However, for asymmetric $A$ one can simply write the regularizer in matrix form
$$\hat{Y}^T L(A)\hat{Y}^T,$$
even though this regularizer with asymmetric $A$ has a less-than-intuitive sum form:
\begin{align*}
\hat{Y}^T L(A)\hat{Y} &= \hat{Y}^TL((A+A^T)/2)\hat{Y} + \frac{1}{2}\hat{Y}^TD(A)\hat{Y} - \frac{1}{2}\hat{Y}^TD(A^T)\hat{Y}\\
&= \frac{1}{2}\sum_{r=1}^T\sum_{s=1}^T A_{rs}(\hat{Y}_r - \hat{Y}_s)^2 + \frac{1}{2}\sum_{r=1}^T \left(\sum_{s=1}^T A_{rs}\right)\hat{Y}_r^2 - \frac{1}{2}\sum_{r=1}^T \left(\sum_{s=1}^T A_{sr}\right)\hat{Y}_r^2.
\end{align*}

\subsection{Bayesian Interpretation of MTA}\label{sec:bayes}
The MTA estimates from (\ref{eqn:mta}) can be interpreted as jointly maximizing the likelihood of $T$ Gaussian distributions with a joint Gaussian Markov random field (GMRF) prior \citep{RueHeld} on the solution.  In MTA, the precision matrix $\Sigma^{-1}$ is $L$, the graph Laplacian of the similarity matrix, and is thus positive semi-definite (and not strictly positive definite); GMRFs with PSD inverse covariances are called intrinsic GMRFs (IGMRFs).

GMRFs and IGMRFs are commonly used in graphical models, wherein the sparsity structure of the precision matrix (which corresponds to conditional independence between variables) is exploited for computational tractability.  Because MTA allows for arbitrary but non-negative similarities between any two tasks, the precision matrix does not (in general) have zeros on the off-diagonal, and it is not obvious how additional sparsity structure of $L$ would be of help computationally.

\section{Simulations}
\label{sec:simulations}
As we have shown in the previous section, MTA is a theoretically rich formulation.  In the next two sections we test the usefulness of constant MTA and minimax MTA given data.

First, we test estimators using simulations so that comparisons to ground truth can be made.  The simulated data was generated from either a Gaussian or uniform hierarchical process with many sources of randomness (detailed below), in an attempt to imitate the uncertainty of real applications, and thereby determine if these are good general-purpose estimators.  The reported results demonstrate that MTA works well when averaged over many different draws of means, variances, and numbers of samples.

Simulations are run for $T = \{2,5,25,500\}$ tasks, and parameters were set so that the variances of the distribution of the true means are the same in both uniform and Gaussian simulations.  Simulation results are reported in Figures \ref{fig:sim_gauss_1} and \ref{fig:sim_gauss_2} for the Gaussian experiments, and Figures \ref{fig:sim_uniform_1} and \ref{fig:sim_uniform_2} for the uniform experiments.   The Gaussian simulations were run as follows:
\begin{enumerate}
\item Fix $\sigma_\mu^2$, the variance of the distribution from which $\{\mu_t\}$ are drawn.
\item For $t=1,\ldots,T$:
    \begin{enumerate}
    \item Draw the mean of the $t$th distribution $\mu_t$ from a Gaussian  with mean 0 and variance $\sigma_\mu^2$.
    \item Draw the variance of the $t$th distribution $\sigma^2_t \sim \text{Gamma}(0.9,1.0) + 0.1$\footnote{The $0.1$ is added to ensure that variance is never zero.}.
    \item Draw the number of samples to be drawn from the $t$th distribution $N_t$ from an integer uniform distribution in the range of $2$ to $100$.
    \item Draw $N_t$ samples $y_{ti} \sim \mathcal{N}(\mu_t,\sigma_t^2)$.
    \end{enumerate}
\end{enumerate}

The uniform simulations were run as follows:
\begin{enumerate}
\item Fix $\sigma_\mu^2$, the variance of the distribution from which $\{\mu_t\}$ are drawn.
\item For $t=1,\ldots,T$:
    \begin{enumerate}
    \item Draw the mean of the $t$th distribution $\mu_t$ from a uniform distribution with mean 0 and variance $\sigma_\mu^2$.
    \item Draw the variance of the $t$th distribution  $\sigma^2_t \sim U(0.1,2.0)$.
    \item Draw the number of samples to be drawn from the $t$th distribution $N_t$ from an integer uniform distribution in the range of $2$ to $100$.
    \item Draw $N_t$ samples  $y_{ti} \sim U[\mu_t - \sqrt{3\sigma_t^2}, \mu_t + \sqrt{3\sigma_t^2}]$.
    \end{enumerate}
\end{enumerate}


We compared constant MTA and minimax MTA to single-task sample averages and to the James-Stein estimator given in (\ref{eqn:jsBock}) (modified to with $T$ instead of the effective dimension).  We also compared to a randomized 5-fold 50/50 cross-validated (CV) version of James-Stein, constant MTA, and minimax MTA.  For the cross-validated versions, we randomly subsampled $N_t/2$ samples and chose the value of $\gamma$ for constant/minimax MTA or $\lambda$ for James-Stein that resulted in the lowest average left-out risk compared to the sample mean estimated with \emph{all} $N_t$ samples.  In the optimal versions of constant/minimax MTA $\gamma$ was set to 1, as this was the case during derivation. Note that the James-Stein formulation with a cross-validated regularization parameter $\lambda$ is simply a convex regularization towards the average of the sample means:
$$\lambda \bar{y}_t + (1-\lambda)\bar{\bar{y}}.$$

We used the following parameters for CV: $\gamma \in \{2^{-5}, 2^{-4},\ldots, 2^5\}$ for the MTA estimators and a comparable set of $\lambda$ spanning $(0,1)$ by the transformation $\lambda = \frac{\gamma}{\gamma + 1}.$  Even when cross-validating, an advantage of using the proposed constant MTA or minimax MTA is that these estimators provide a data-adaptive scale for $\gamma$, where $\gamma = 1$ sets the regularization parameter to be $\frac{a^*}{T}$ or $\frac{1}{T(b_u-b_l)^2}$, respectively.

\newpage
\begin{figure}[h!]
\begin{center}
\begin{tabular}{c}
$\bf{Gaussian, T=2}$ \\
\includegraphics[width=0.87\textwidth]{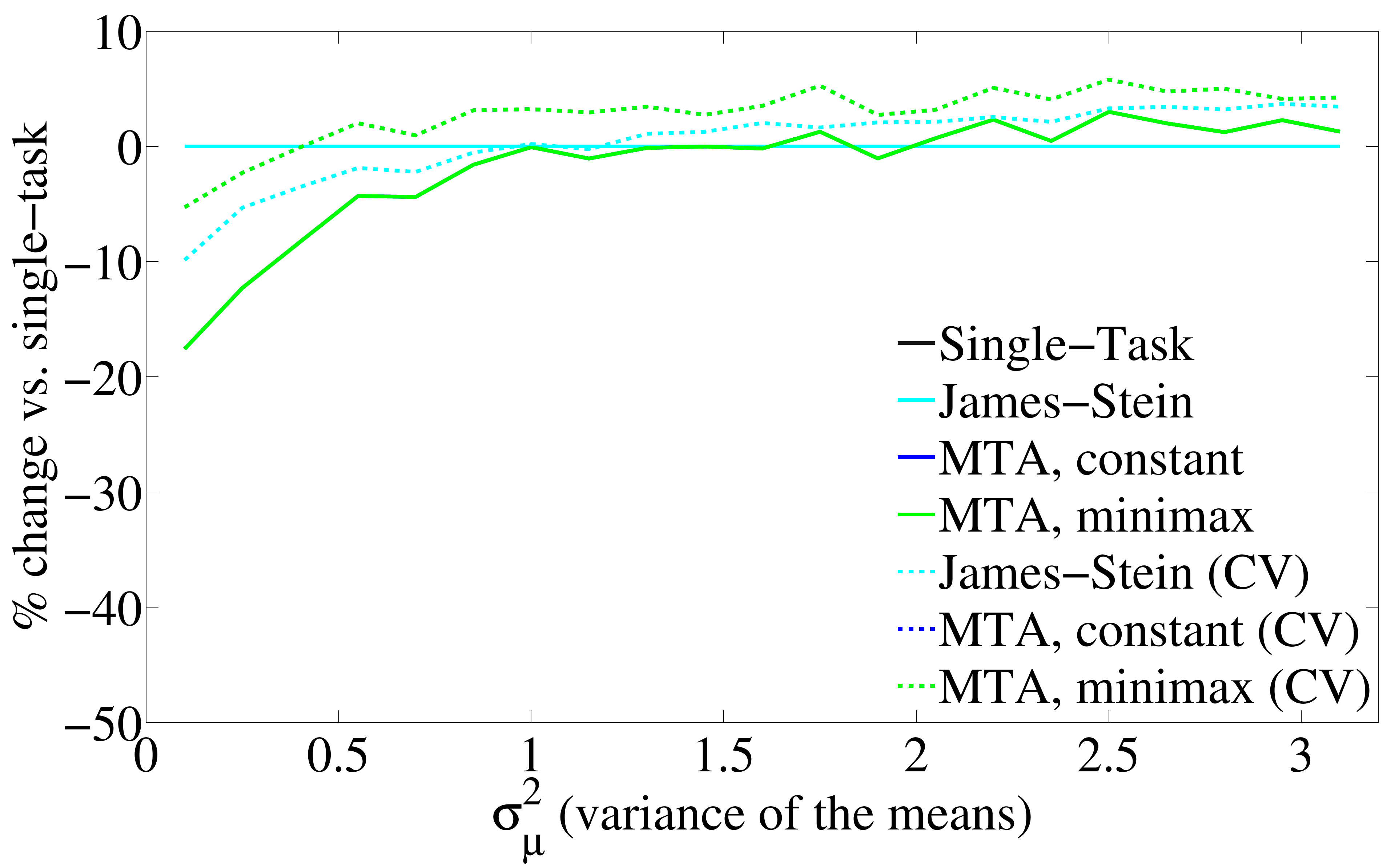}\\
$\bf{T=5}$\\
\includegraphics[width=0.87\textwidth]{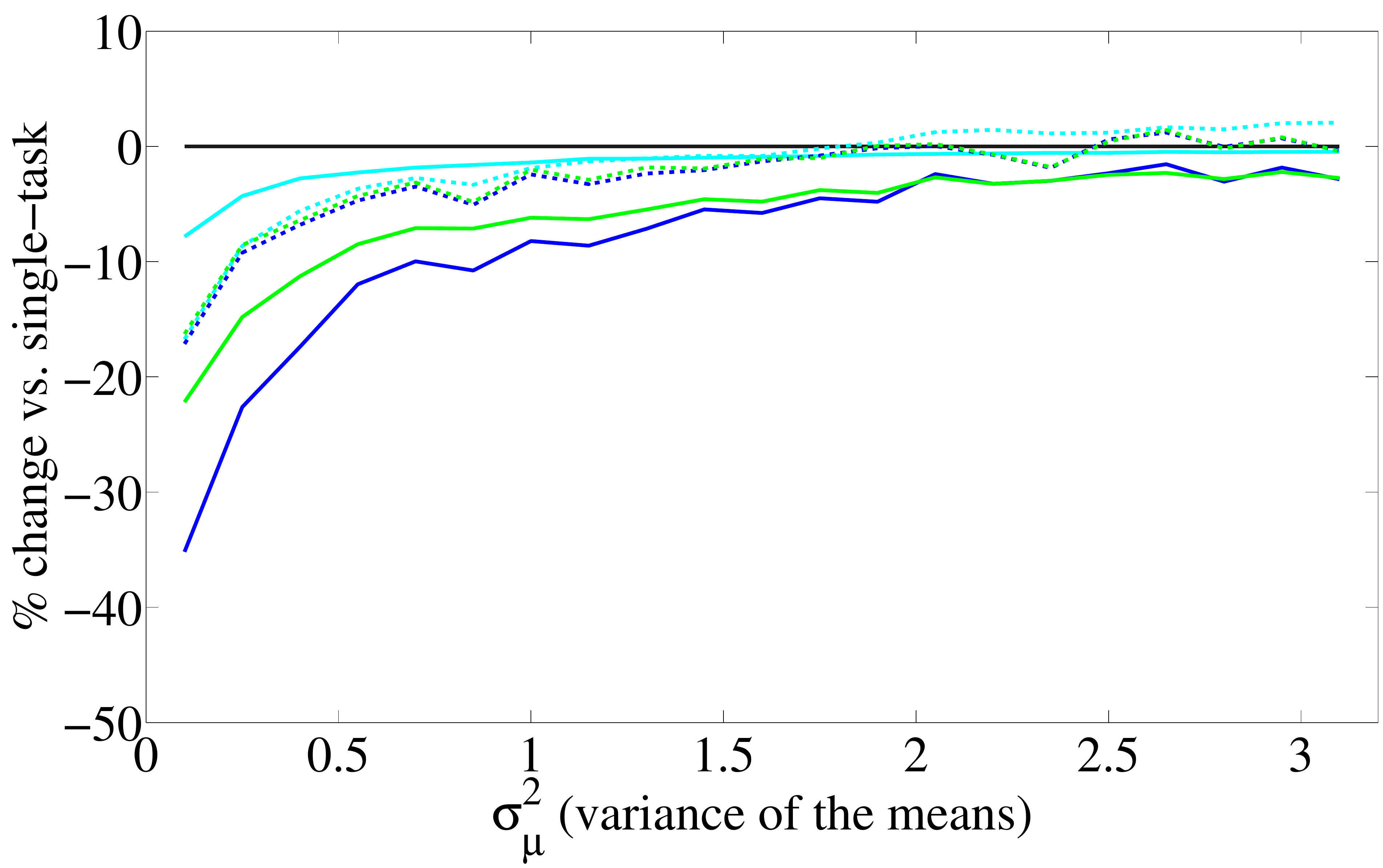}
\end{tabular}
\caption{\footnotesize Gaussian experiment results for $T = \{2,5\}$.  The y-axis is average (over 10000 random draws) percent change in risk vs. single-task, such that $-50$\% means the estimator has half the risk of single-task.  Note: for $T=2$ the James-Stein estimator reduces to single-task, and so the cyan and black lines overlap.  Similarly, for $T=2$, constant MTA and minimax MTA are identical, and so the blue and green lines overlap.}\label{fig:sim_gauss_1}
\end{center}
\end{figure}
\newpage
\begin{figure}[h!]
\begin{center}
\begin{tabular}{c}
$\bf{Gaussian, T=25}$ \\
\includegraphics[width=0.87\textwidth]{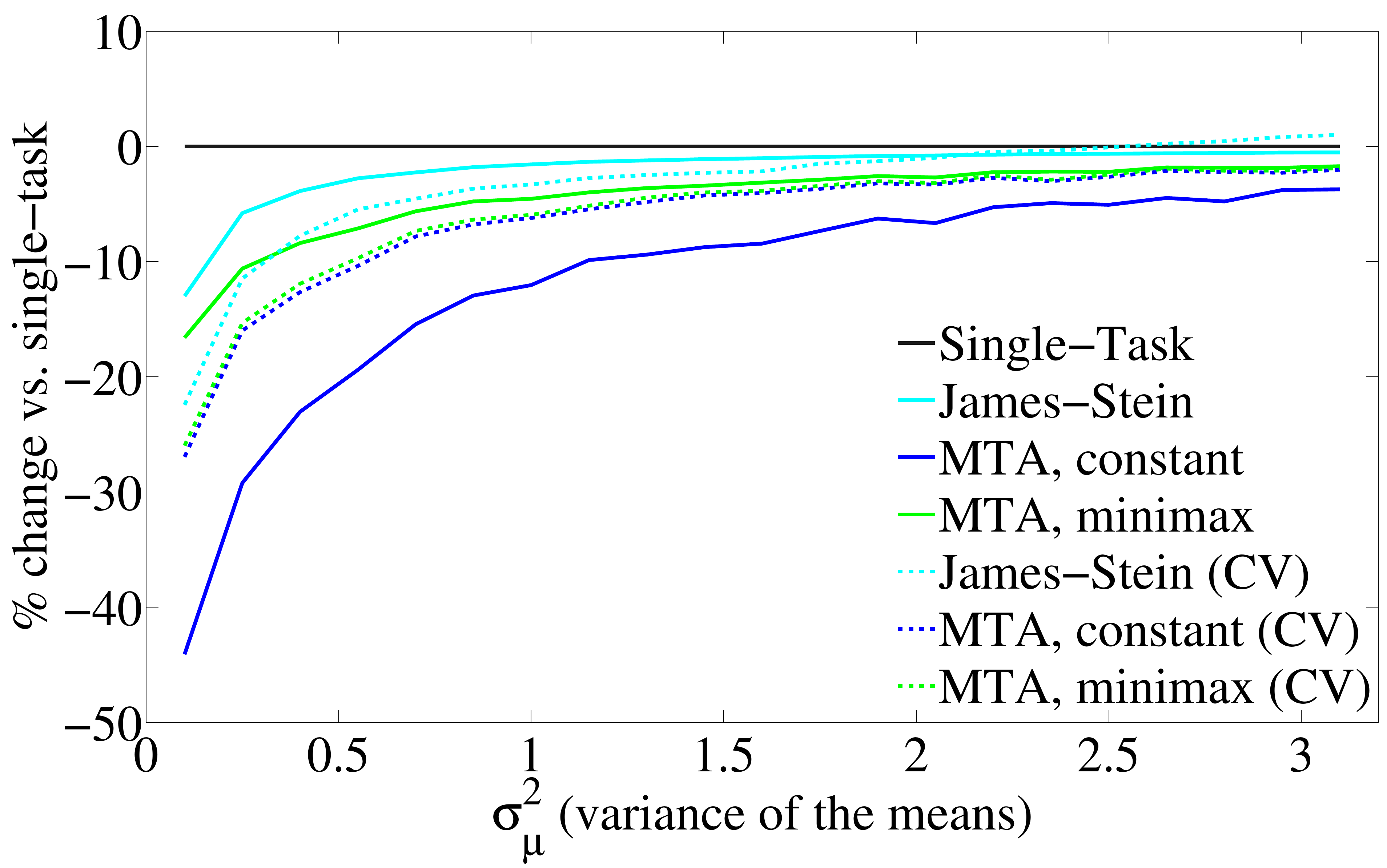}\\
$\bf{T=500}$\\
\includegraphics[width=0.87\textwidth]{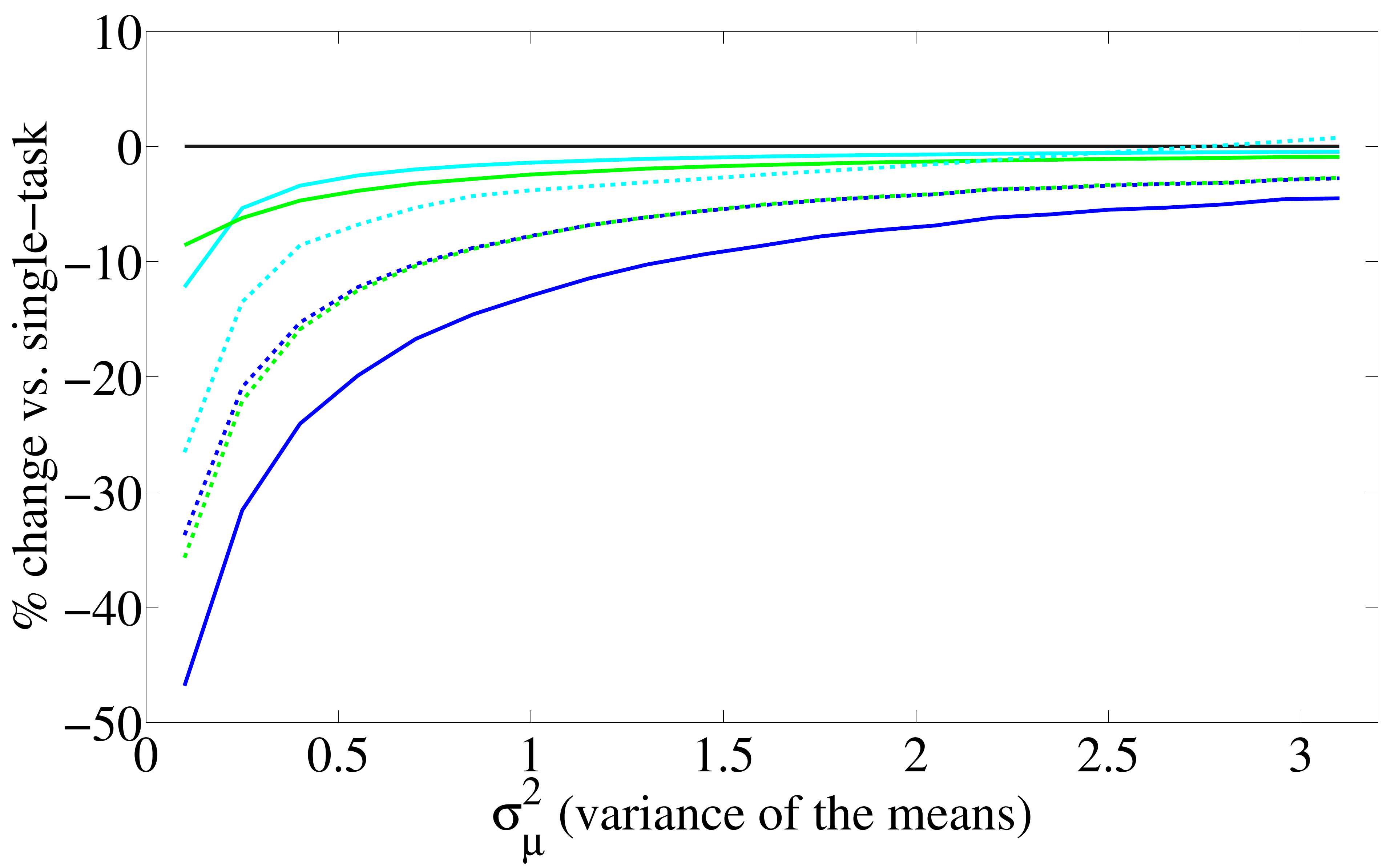}
\end{tabular}
\caption{\footnotesize  Gaussian experiment results for $T = \{25,500\}$.  The y-axis is average (over 10000 random draws) percent change in risk vs. single-task, such that $-50$\% means the estimator has half the risk of single-task.}\label{fig:sim_gauss_2}
\end{center}
\end{figure}

\newpage

\begin{figure}[h!]
\begin{center}
\begin{tabular}{c}
$\bf{Uniform, T=2}$ \\
\includegraphics[width=0.87\textwidth]{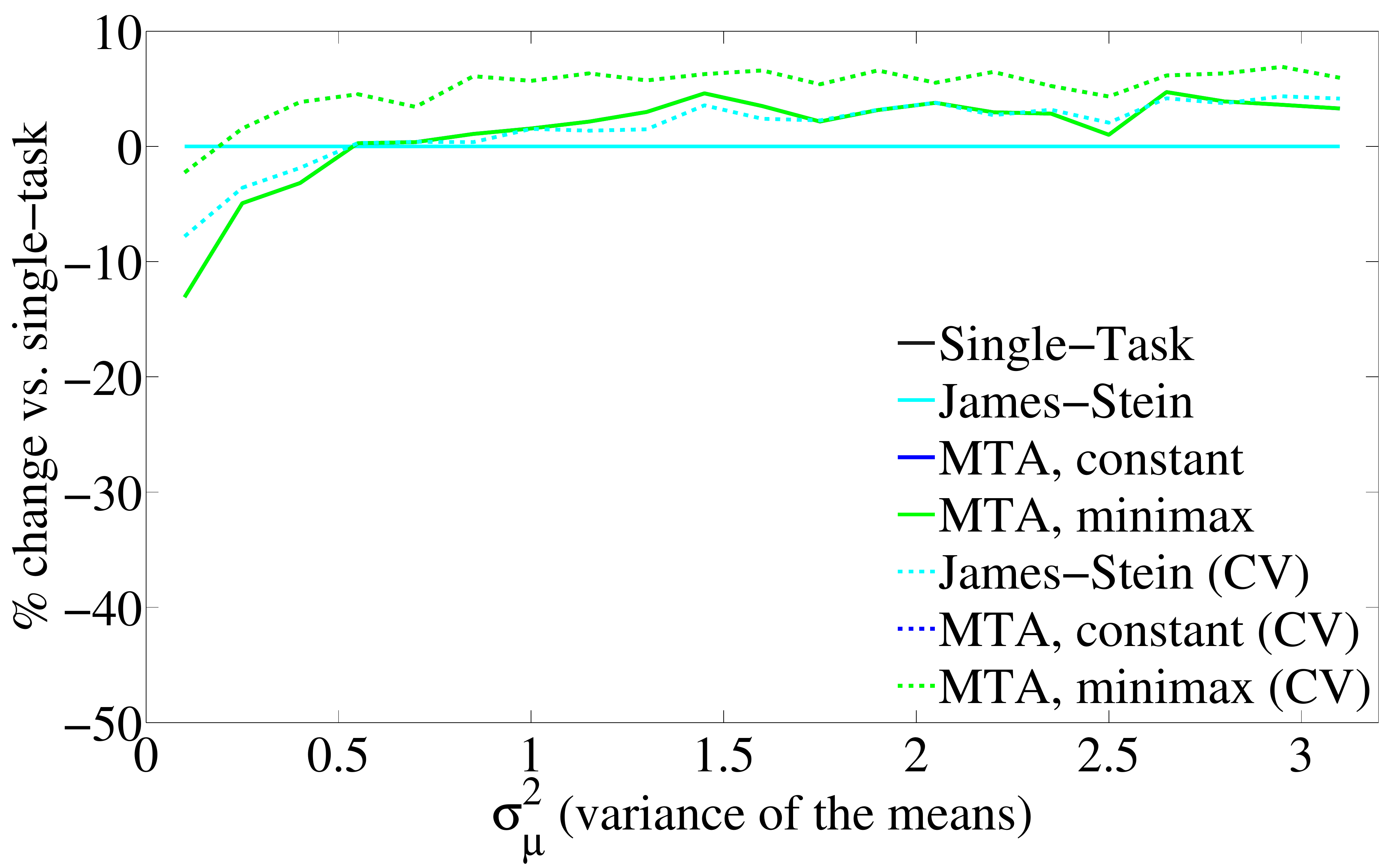}\\
$\bf{T=5}$\\
\includegraphics[width=0.87\textwidth]{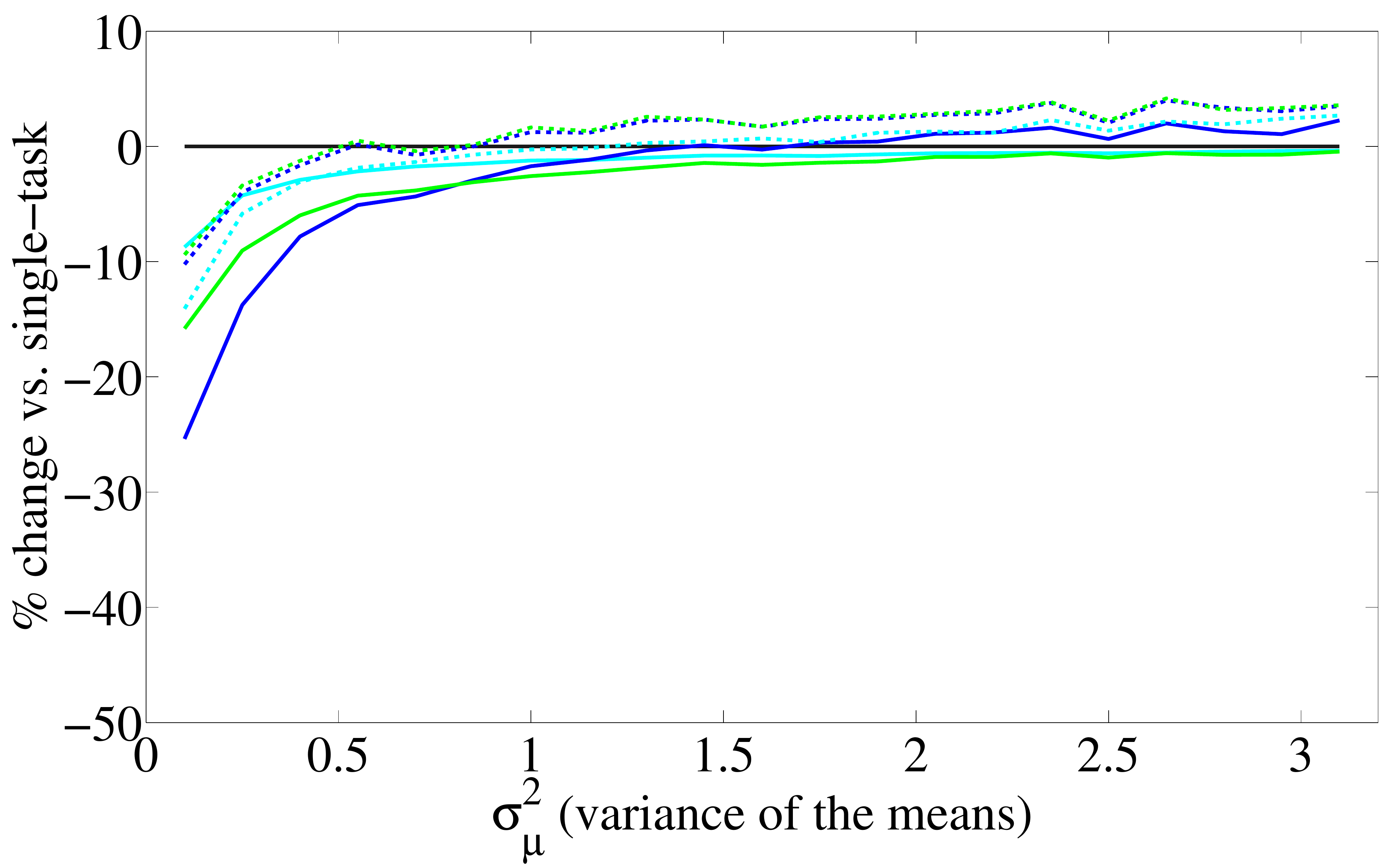}
\end{tabular}
\caption{\footnotesize Uniform experiment results for $T = \{2,5\}$.  The y-axis is average (over 10000 random draws) percent change in risk vs. single-task, such that $-50$\% means the estimator has half the risk of single-task.  Note: for $T=2$ the James-Stein estimator reduces to single-task, and so the cyan and black lines overlap.  Similarly, for $T=2$, constant MTA and minimax MTA are identical, and so the blue and green lines overlap.}\label{fig:sim_uniform_1}
\end{center}
\end{figure}

\newpage

\begin{figure}[h!]
\begin{center}
\begin{tabular}{c}
$\bf{Uniform, T=25}$ \\
\includegraphics[width=0.87\textwidth]{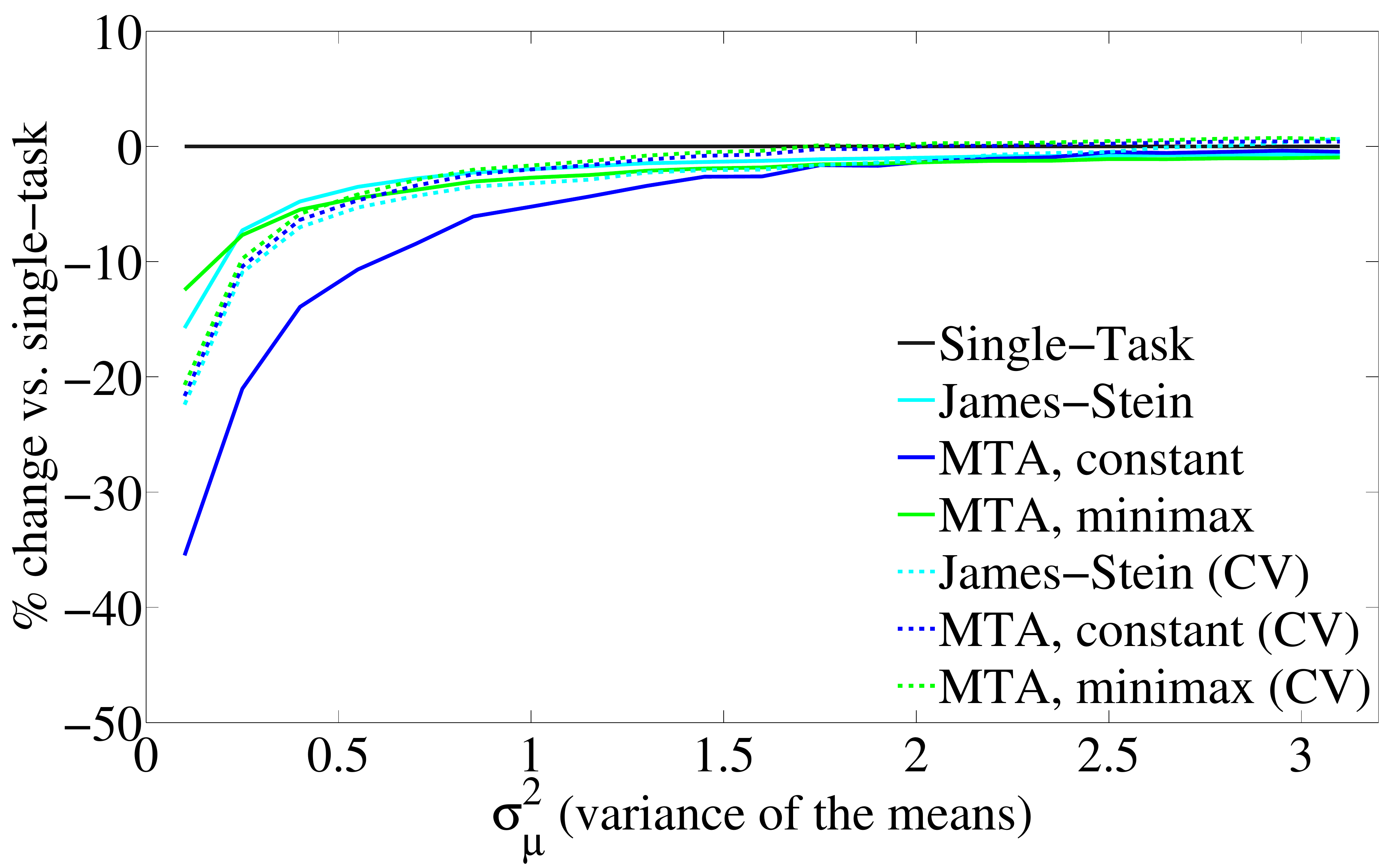}\\
$\bf{T=500}$\\
\includegraphics[width=0.87\textwidth]{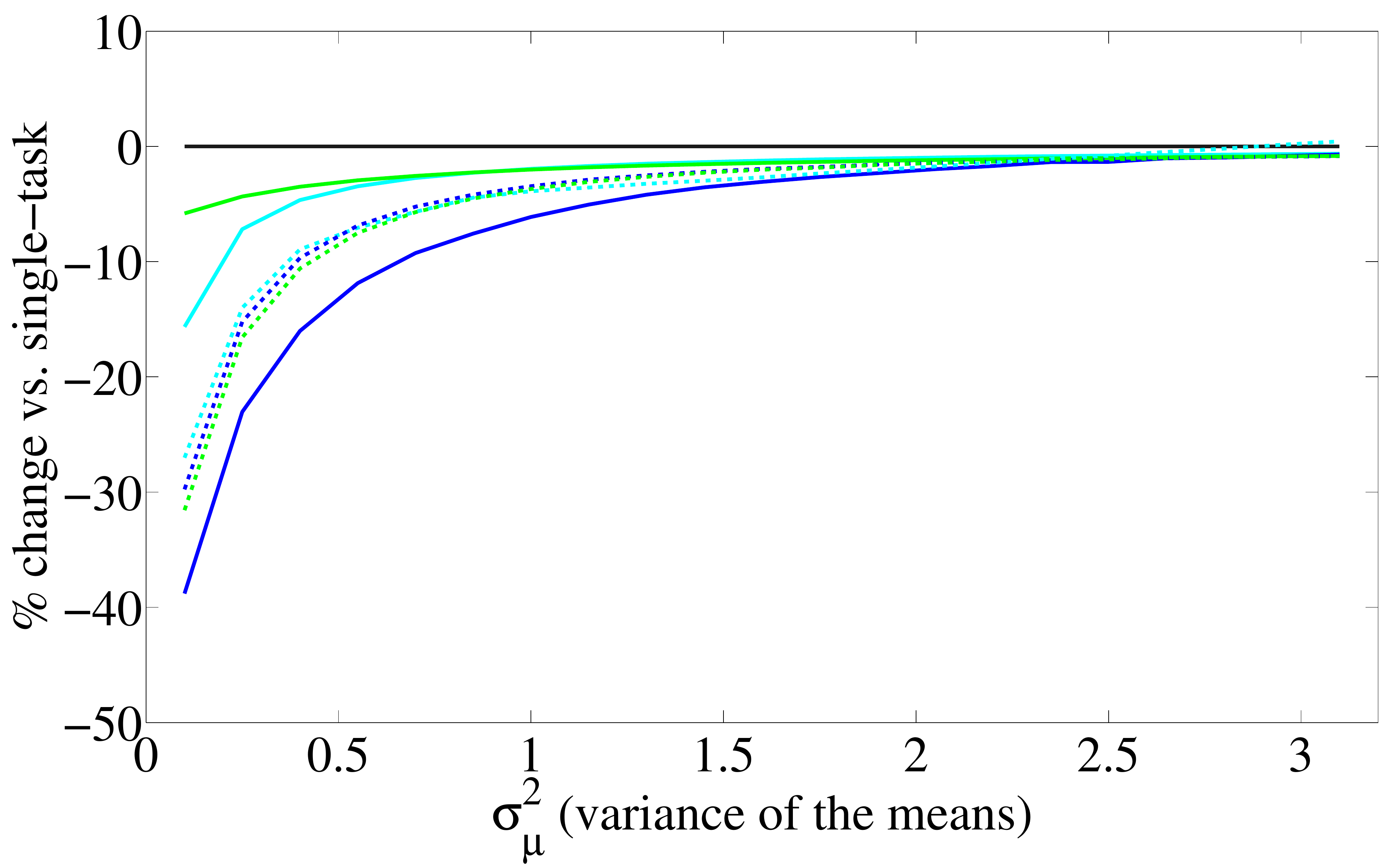}
\end{tabular}
\caption{\footnotesize  Uniform experiment results for for $T = \{25,500\}$.  The y-axis is average (over 10000 random draws) percent change in risk vs. single-task, such that $-50$\% means the estimator has half the risk of single-task.}\label{fig:sim_uniform_2}
\end{center}
\end{figure}

\newpage


Some observations from Figures \ref{fig:sim_gauss_1}-\ref{fig:sim_uniform_2}:
\begin{itemize}
\item Further to the right on the x-axis the means are more likely to be further apart, and multi-task approaches help less on average.
\item For $T=2$, the James-Stein estimator reduces to the single-task estimator.  The MTA estimators provide a gain while the means are close with high probability (that is, when $\sigma^2_\mu < 1$) but deteriorate quickly thereafter.
\item For $T=5$, constant MTA dominates in the Gaussian case, but in the uniform case does worse than single-task when the means are far apart.  Note that for all $T>2$ minimax MTA almost always outperforms James-Stein and always outperforms single-task, which is to be expected as it was designed conservatively.
\item For $T=25$ and $T=500$, we see the trend that all estimators benefit from an increase in the number of tasks.  The difference between $T=25$ performance and $T=500$ performance is minor, indicating that benefit from further tasks levels off early on.
\item For constant MTA, cross-validation is always worse than the estimated optimal regularization, while the opposite is true for minimax MTA.  This is to be expected, as minimax estimators are not designed to minimizes the average risk, which is what we report and the metric optimized during cross-validation.
\item Since both constant MTA and minimax MTA use a similarity matrix of all ones scaled by a constant (albeit it a different one for constant MTA and minimax MTA), cross-validating over a set of possible $\gamma$ should result in similar performance, and this can be seen in the Figures (i.e. the green and blue dotted lines are superimposed).
\end{itemize}

In summary, when the tasks are close to each other compared to their variances, constant MTA is the best estimator to use by a wide margin.  When the tasks are farther apart, minimax MTA  provides a win over both James-Stein and sample averages.

\subsection{Oracle Performance}
To illustrate the best achievable performance with MTA, Figure \ref{fig:oracleresults} shows the effect of using the true ``oracle'' means and variances for the calculation of optimal pairwise similarities.  This experiment separates separates how well the MTA formulation can do from the issue of estimating the optimal similarity matrix from the data.  We use the pairwise oracle matrix $A$\footnote{After experimentation, we found that this similarity matrix gave us the best oracle performance.  Note that an estimated optimal similarity $\hat{A}_{rs} = \frac{2}{(\bar{y}_r-\bar{y}_s)^2}$ almost always does worse than constant MTA and minimax MTA.}:
$$A^\text{orcl}_{rs} = \frac{2}{(\mu_r-\mu_s)^2},$$
which consistently bested oracle constant MTA and oracle minimax MTA.  The plot reproduces the results from the $T=5$ Gaussian simulation (excluding cross-validation results), and includes the performance of oracle pairwise MTA.  Oracle MTA is over 30\% better than constant MTA, indicating that practical estimates of the similarity, while improving on single-task estimation, are highly suboptimal compared to possible MTA performance.

\begin{figure}[th!]
\begin{center}
\includegraphics[width=1.0\textwidth]{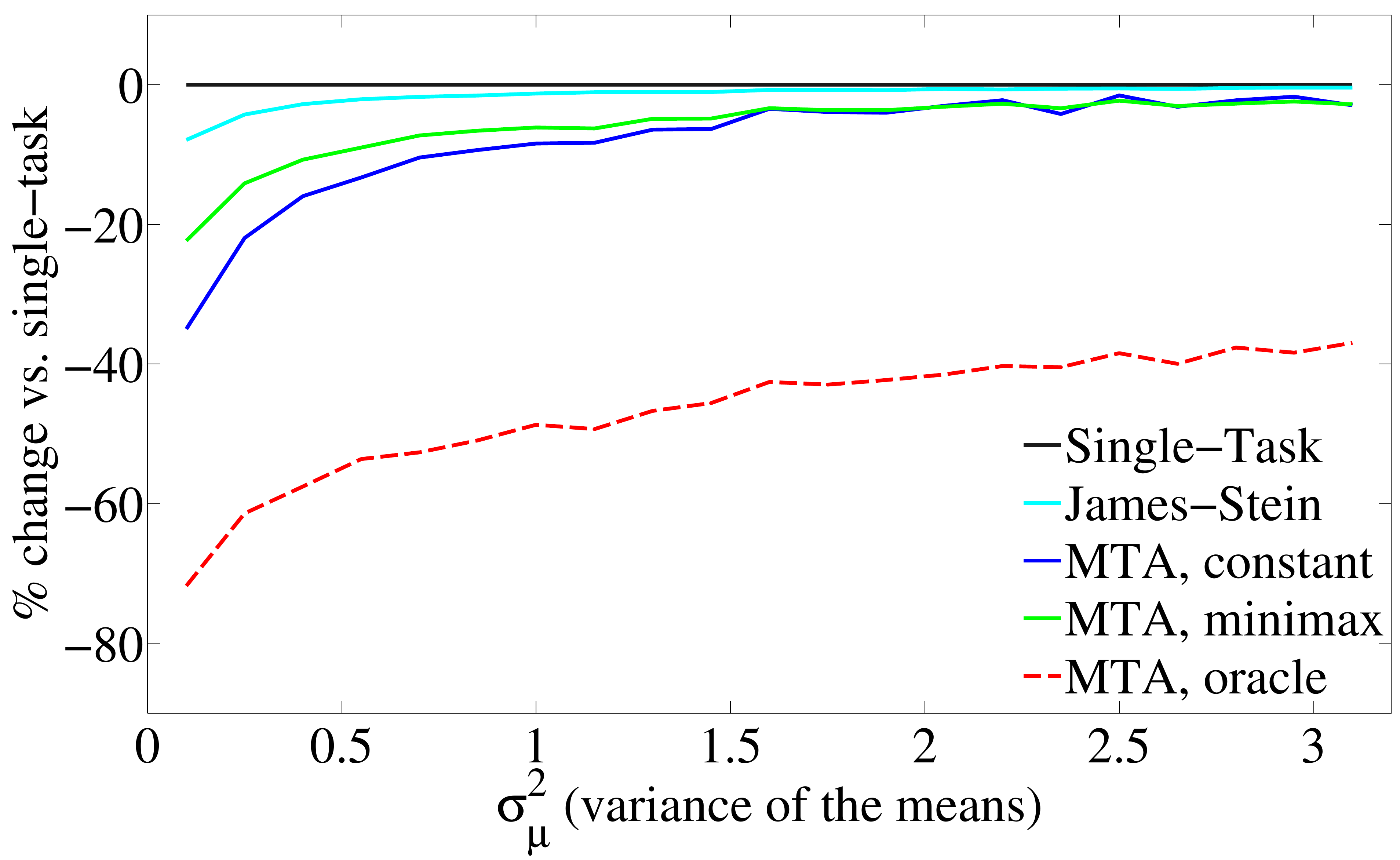}
\caption{\footnotesize Average (over 10000 random draws) percent change in risk vs. single-task with $T=5$ for the Gaussian simulation.  Oracle MTA uses the true means and variance to specify the weight matrix $W$.}\label{fig:oracleresults}
\end{center}
\end{figure}

\newpage
\section{Applications}\label{sec:applications}
We present three applications with real data. The first two applications parallel the simulations: estimating expected values of final grades and sales of related products. The third application uses MTA for multi-task kernel density estimation, highlighting the applicability of MTA to any algorithm that uses sample averages.

\subsection{MTA for Grade Estimation}\label{sec:grades}
The goal of this application is to predict the final class grades $\{\mu_t\}_{t=1}^T$ for all $T$ students, given only each student's $N$ homework grades $\{y_{ti}\}_{i=1}^N$ (in this application $N_t = N$ for all $t$ as every student had been assigned the same number of homeworks).  The final class grades include homeworks, projects, labs, quizzes, midterms, and the final exam, but only the homework grades are used to predict the final grade.  The 16 anonymized datasets were provided by instructors at the University of Washington Department of Electrical Engineering.
Some experimental details:
\begin{itemize}
\item Each of the 16 datasets (classes) constitutes a single experiment, and the students in that class are treated as the tasks.
\item All the grades have been normalized to be between $0$ and $100$.
\item Homeworks that were never handed in were assigned $0$ points.
\item The number of students across the 16 classes is between $T=16$ and $T=149$.
\item Cross-validation parameters were chosen by training on $N/2$ of the homework grades and validating on the sample mean of all $N$ given grades.  Again, we used randomized 5-fold 50/50 cross-validation.
\item For each class, a single pooled variance estimate was used for all tasks (that is, students).  In other words $\sigma_t^2 = \sigma^2,$ for all $t$.
\item The estimator marked ``one-task' is just a constant pooled mean for all tasks:
$$\hat{y}^\text{pl}_t = \frac{1}{TN}\sum_{t=1}^T \sum_{i=1}^N y_{ti}.$$
\item For each class of students, the error measurement for estimator $\hat{y}$ is the risk (average of squared errors) across all $T$ students:
$$\frac{1}{T}\sum_{t=1}^T (\mu_t-\hat{y}_t)^2.$$
\end{itemize}

This error metric was computed for each class (dataset) separately, and the percent change in average risk vs. single-task are reported in Table \ref{tab:graderesults}.

\begin{table}[h!]
\caption{Percent change in risk vs. single-task for the grade estimation application (lower is better).  `JS' denotes James-Stein, `MTA $\hat{a}^*$' and `MTA mm' denote constant MTA and minimax MTA, respectively, `CV' denotes cross-validation, and `STD' denotes standard deviation.  Lower is better.}
\label{tab:graderesults}
\begin{center}
\begin{tabular}{l|rrrrrrr}
 Class & One-Task& JS & JS  & MTA & MTA & MTA  & MTA \\
 Size &         &    &   & $\hat{a}^*$    & $\hat{a}^*$ &   mm  & mm \\
 &         &    & CV  &     & CV &       & CV \\
 \hline
 16 & $26.3$ & $0.7$ & $0.2$ & $0.6$ & $0.1$ & $\bf{0}$ & $0.1$ \\
 29 & $-6.8$ & $-11.0$ & $\bf{-13.4}$ & $-10.8$ & $-5.9$ & $-1.7$ & $-6.4$ \\
 36 & $\bf{-28.3}$ & $-17.4$ & $-12.4$ & $-16.0$ & $-9.1$  & $-2.8$ & $-10.0$ \\
 39 & $42.0$ & $\mathbf{-5.8}$ & $-2.3$ & $-5.6$ & $-0.9$ & $-0.9$ & $-0.9$ \\
 44 & $3.0$ & $\bf{-47.6}$ & $-47.3$ & $-42.7$ & $-42.7$ & $-7.0$ & $-41.1$ \\
 47 & $\bf{-12.8}$ & $-8.0$ & $-5.2$ & $-7.1$ & $-4.1$ & $-0.7$ & $-2.6$ \\
 48 & $\mathbf{-21.0}$ & $-20.5$ & $-13.7$ & $-18.5$ & $-5.8$ & $-2.5$ & $-4.9$ \\
 50 & $63.5$ & $63.5$ & $16.3$ & $9.3$ & $9.3$ & $\bf{-4.4}$ & $14.8$ \\
 50 & $3.7$ & $\bf{-33.6}$ & $-19.1$ & $-29.7$ & $-10.1$ & $-3.2$ & $-11.6$ \\
 57 & $23.3$ & $-3.8$ & $\bf{-4.1}$ & $-3.6$ & $-2.1$ & $-0.4$ & $-1.4$ \\
 58 & $-0.2$ & $\mathbf{-16.3}$ & $-5.9$ & $-15.6$ & $-4.4$ & $-2.8$ & $-5.4$ \\
 68 & $-16.9$ & $\mathbf{-45.5}$ & $-38.5$ & $-39.0$ & $-27.7$ & $-6.1$ & $-31.1$ \\
 69 & $-14.7$ & $-41.0$ & $\bf{-42.4}$ & $-39.8$ & $-39.8$ & $-4.5$ & $-35.7$ \\
 72 & $34.6$ & $\bf{-32.9}$ & $-29.4$ & $-29.0$ & $-18.3$ & $-4.0$ & $-14.1$ \\
 110 & $5.7$ & $\mathbf{-14.8}$ & $-11.5$ & $-13.4$ & $-7.7$ & $-1.2$ & $-8.6$ \\
 149 & $\mathbf{-16.6}$ & $-11.7$ & $-11.8$  & $-10.1$  & $-5.9$  & $-0.8$  & $-5.8$  \\
 \hline
 & $5.3$ & $-15.4$ & $-15.0$ & $\bf{-16.9}$ & $-10.9$ & $-2.7$ & $-10.3$ \\
 & \emph{$25.9$} & \emph{$25.9$} & \emph{$16.9$} & \emph{$15.2$} & \emph{$14.3$} & \emph{$2.1$} & \emph{$14.4$}
\end{tabular}
\end{center}
\end{table}


Some observations:
\begin{itemize}
\item Constant MTA (without CV) has the lowest percent change averaged across all classes.
\item The James-Stein estimator has the best percent change from single-task on 7 of the 16 classes.
\item The cross-validated versions of estimators do worse than their estimated optimal counterparts.
\item For classes 5 and 8, only minimax MTA with estimated similarity does better than the single-task estimate $\bar{y}$.  This is rare, but not impossible; a few of the students happened to have a low average homework grade and an even lower final grade, which resulted in an outsized contribution to the risk.  The expectation is over random samples used to form the estimate, and here the particular realizations of the all the grades were poor.  Also, the model that homework grades are drawn iid from a distribution with the final grade as the mean is only a model\footnote{``All models are wrong, but some are useful." --George E. P. Box}.
\item Minimax MTA was never worse than single-task, robustly providing relatively small gains, as designed.
\item The James-Stein estimator is also a minimax estimator, but its performance is as highly variable as the one-task estimator.  This is because of the positive-part aspect of the JS estimator -- when the positive-part boundary is triggered, JS reverts to the one-task estimator.
\item Surprisingly, the one-task estimator, which pools all students' scores to estimate a single grade, does better than single-task for half of the classes, and is the best performer for 4 out of 16.  When the one-task estimator outperforms single-task, we hypothesize that individual homework grades are poor estimates of final grades.  Further, when the one-task estimator is the best estimator, we hypothesize that the assumed model is wrong.  That is, the homework grades are \emph{not} iid draws from the ``true" distribution of grades, and, in fact, in those case the homework grades of any individual student provide little information about the final grade.  This may occur if the instructor chose to put a small weight on the homework grades, or if the tests and labs required a different skill set from the homework.
\end{itemize}

\subsection{Application: Estimating Product Sales}\label{sec:product_sales}
We now consider two multi-task problems using sales data from Artifact Puzzles. For both problems, we model the given samples as being drawn iid from each task.

The first problem estimates the impact of a particular \emph{puzzle} on repeat business: ``Estimate how much any customer will spend on a given order, if on their last order they purchased the $t$th puzzle, for each of $T = 77$ puzzles." The samples were the amounts different customers had spent on orders after buying each of the $t$ puzzles during a given time period, and ranged from $0$ for customers that had not re-ordered in the specified time period, to $480$. The number of samples for each puzzle ranged from $N_t = 8$ to $N_t = 348$.

The second problem estimates the monetary value of a particular \emph{customer}: ``Estimate how much the $t$th customer will spend on a given order, for each of $T = 477$ customers." The samples were the order amounts for each of the $T$ customers. Order amounts varied from $15$ to $480$. The number of samples for each customer ranged from $N_t = 2$ to $N_t = 17$.

We have only samples, no ground truth, so to compare the estimators we treat the single-task means  computed from all of the samples as the ground truth, and compare to estimates computed from a uniformly randomly chosen $50\%$ of the samples. Results in Table \ref{tab:results} are averaged over 1000 random draws of the $50\%$ used for estimation.  Again, we used 5-fold cross-validation with the same parameter choices as in the simulations section.

We bolded those entries that were the best or not statistically significantly better than the best according to two one-sided Wilcoxon rank statistical  significance tests.

Some observations:
\begin{itemize}
\item One-task is a very poor estimator for all of the experiments in this section.
\item Using cross-validation with the two minimax estimators (James-Stein and minimax MTA) statistically significantly outperformed their estimated optimal counterparts.  This is consistent with the simulation results.
\item Constant MTA provided comparable performance to the cross-validated estimators.  It was the best or not statistically significantally better than the best of all the other non-CV estimators.
\end{itemize}

\begin{table}[h!]
\caption{Percent change in average risk for puzzle and customer data (first two columns, lower is better), and mean reciprocal rank for terrorist data (last column,  higher is better).}
\label{tab:results}
\begin{center}
\begin{tabular}{l||rr||r}
\hline
Estimator         & Puzzles        & Customers          &  Suicide Bombings \\
                  & $T = 77$       & $T=477$            &   $T = 7$         \\
\hline
Single-Task &  0     &    0       &  0.15   \\
One-Task &  181.7     &    109.2       &  0.13   \\
James-Stein       &  -6.9        &  -14.0         &  0.15        \\
James-Stein CV  &  -21.2       &  -31.0         &  0.19       \\
Constant MTA      &  -17.5       &  -\bf{32.3}    &  0.19       \\
Constant MTA CV &  -\bf{21.7}  &  -30.9         &  0.19     \\
Minimax MTA       &  -8.4        &  -3.0          &  0.19   \\
Minimax MTA CV  &  -19.8       &  -25.0         &  0.19 \\
Expert MTA        &   -             &    -            &  0.19  \\
Expert MTA CV   &   -             &    -            &  0.19  \\
\end{tabular}
\end{center}
\end{table}

\subsection{Density Estimation for Terrorism Risk Assessment} \label{sec:KDE}
In this section we present \emph{multi-task kernel density estimation} (MT-KDE), a variant of MTA.

MTA can be used whenever averages are taken.  Recall that for standard single-task kernel density estimation (KDE) \citep{Silverman}, a set of random samples
$x_i \in \R^d, i \in \{1,\ldots,N\}$ are assumed to be iid from an unknown distribution $p_X$,
and the problem is to estimate the density for a query sample, $z \in \R^d$.
Given a kernel function $K(x_i,x_j)$, the un-normalized single-task
KDE estimate is $$\hat{p}(z) = \frac{1}{N}\sum_{i=1}^N K(x_i,z),$$
which is just a sample average.

When multiple kernel densities $\{\hat{p}_t(z)\}_{t=1}^T$ are estimated for the same domain, we replace the multiple sample averages with MTA estimates, which we refer to as multi-task kernel density estimation (MT-KDE).

We compared KDE and MT-KDE on a problem of estimating the probability of terrorist events in
Jerusalem using the Naval Research Laboratory's Adversarial Modeling and Exploitation Database (NRL AMX-DB).
The NRL AMX-DB combined multiple open primary sources\footnote{Primary sources included
 the NRL Israel Suicide Terrorism Database (ISD) cross referenced with
 open sources (including the Israel Ministry of Foreign Affairs, BBC, CPOST,
 Daily Telegraph, Associated Press, Ha'aretz Daily, Jerusalem Post, Israel National News), as well as
 the University of New Haven Institute for the Study of Violent Groups, the University of
 Maryland Global Terrorism Database, and the National Counter Terrorism Center
 Worldwide Incident Tracking System.}
to create a rich representation of the geospatial features of urban Jerusalem and the surrounding region, and accurately
geocoded locations of terrorist attacks.  Density estimation models are used to analyze the behavior
of such violent agents, and to allocate security and medical resources.
In related work, \citep{Brown:04} also used a Gaussian kernel density estimate to assess risk from past terrorism events.

The goal in this application is to estimate a risk density for 40,000 geographical locations (samples) in a 20km $\times$ 20km area of interest in Jerusalem.  Each geographical location is represented by a $d=76$-dimensional feature vector.  Each of the 76 features is the distance in kilometers to the nearest instance of some geographic location of interest, such as the nearest market or bus stop. Locations of past events are known for 17 suicide bombings.  All the events are attributed to one of seven terrorist groups. The density estimates for these seven groups are expected to be related, and are treated as $T=7$ tasks.

The kernel $K$ was taken to be a Gaussian kernel with identity covariance; the bandwidth was set to 1.  In addition to constant $A$ and minimax $A$, we also obtained a side-information $A$ from terrorism expert Mohammed M. Hafez of the Naval Postgraduate School; he assessed the similarity between the seven groups during the Second Intifada (the time period of the data), providing similarities between $0$ and $1$.  The similarities are shown in Table \ref{tab:FactionSimilarities}.

\begin{table}
\caption{Hafez's Similarity Matrix $A$}
\label{tab:FactionSimilarities}
\begin{tabular}{lccccccc}
\hline\noalign{\smallskip}
          & AAMB & Hamas & PIJ	& PFLP 	& Fatah & Force17 & Unknown \\
\noalign{\smallskip}\hline\noalign{\smallskip}
AAMB      &  0& .2& .2& .6& .8& .8& .6 \\
Hamas     & .2&  0& .8& .2& .2& .2& .4 \\
PIJ       & .2& .8&  0& .2& .2& .2& .4 \\
PFLP	  & .6& .2& .2&  0& .6& .6& .5 \\
Fatah	  & .8& .2& .2& .6&  0&  1& .6 \\
Force17   & .8& .2& .2& .6&  1&  0& .6 \\
Unknown   & .6&	.4&	.4&	.5&	.6&	.6&	0  \\
\noalign{\smallskip}\hline
\end{tabular}
\end{table}

The KDE estimates were computed separately for each grid point and each task.  The MT-KDE estimates were obtained for one grid point at a time, but for all of the tasks simultaneously.  In other words, the regularization was performed only across tasks, and not across grid points.

Leave-one-out cross validation was used to assess KDE and MT-KDE for this problem, as follows.
After computing the KDE and MT-KDE density estimates using all but one of the training examples $\{x_{ti}\}$ for each task, we sort the resulting 40,000 estimated probabilities for each of the seven tasks, and extract the rank of the left-out known event.  The mean reciprocal rank (MRR) metric is reported in Table \ref{tab:results}.  Ideally, the MRR of the left-out events would be as close to $1$ as possible, and indicating that the location of the left-out event is at high-risk.  The results show that the MRR for MT-KDE are lower or not worse than those for KDE for both problems; there are, however, too few samples to verify statistical significance of these results.  Also, note that the solution of pooling all of the training data into one big task gives inferior performance, and we suspect that this is because each terrorist group has its own target preferences.

\section{Summary}
Though perhaps unintuitive, we showed that both in theory and in practice estimating multiple \emph{unrelated} means in a joint MTL fashion can improve the overall risk, even more so than the classic, battle-tested James-Stein estimator.  Averaging is common, and MTA has potentially broad applicability as a subcomponent to many algorithms, such as k-means clustering, kernel density estimation, or non-local means denoising.

\section*{Acknowledgments}
We thank Peter Sadowski, and Carol Chang, Brian Sandberg, and Ruth Willis of the Naval Research Lab for the terrorist event dataset and helpful discussions. We thank Mohammed M. Hafez of the Naval Postgraduate School for the matrix of similarities of the terrorist groups. 
This work was funded by a United States PECASE Award and by the United States Office of Naval Research.

\appendix

\section*{Appendix A: MTA Closed-form Solution}
When all $A_{rs}$ are non-negative, the differentiable MTA objective is convex, and a admits closed-form solution.  First, we rewrite the objective in (\ref{eqn:mta}) in matrix notation:
\begin{align*}
&\frac{1}{T} \sum_{t=1}^T \frac{1}{\sigma_t^2} \sum_{i=1}^{N_t} (y_{ti} - \hat{y}_t)^2 + \frac{\gamma}{T^2} \sum_{r=1}^T \sum_{s=1}^T A_{rs}(\hat{y}_r-\hat{y}_s)^2\\
&= \frac{1}{T}\sum_{t=1}^T \frac{1}{\sigma_t^2} \sum_{i=1}^{N_t} (y_{ti} - \hat{y}_t)^2 + \frac{\gamma}{T^2} \hat{y}^T L \hat{y}\\
&= \frac{1}{T} \sum_{t=1}^T \frac{1}{\sigma_t^2} \sum_{i=1}^{N_t} \left( y_{ti}^2 + \hat{y}_t^2 - 2y_{ti}\hat{y}_t\right) + \frac{\gamma}{T^2} \hat{y}^T L \hat{y}\\
&= \frac{1}{T} \sum_{t=1}^T \left( \frac{1}{\sigma_t^2} \sum_{i=1}^{N_t} y_{ti}^2 + \frac{1}{\sigma_t^2}\hat{y}_t^2 \sum_{i=1}^{N_t}1  - 2 \frac{1}{\sigma_t^2}\hat{y}_t \sum_{i=1}^{N_t} y_{ti}\right) + \frac{\gamma}{T^2} \hat{y}^T L \hat{y}\\
&= \frac{1}{T} \sum_{t=1}^T \left( \frac{1}{\sigma_t^2} \sum_{i=1}^{N_t} y_{ti}^2 + \frac{ N_t}{\sigma_t^2}\hat{y}_t^2 - 2 \frac{N_t}{\sigma_t^2}\hat{y}_t  \bar{y}_t \right) + \frac{\gamma}{T^2} \hat{y}^T L  \hat{y} \\
&= \frac{1}{T} \left( \sum_{t=1}^T \frac{1}{\sigma_t^2} \sum_{i=1}^{N_t} y_{ti}^2 + \hat{y}^T \Sigma^{-1}\hat{y} - 2 \hat{y}^T \Sigma^{-1}\bar{y}\right) + \frac{\gamma}{T^2} \hat{y}^T L \hat{y},
\end{align*}
where $L = D - (A+A^T)/2$  is the graph Laplacian matrix $(A+A^T)/2$, $\Sigma$ is a diagonal matrix with $\Sigma_{tt} = \frac{\sigma_t^2}{N_t}$, and $\hat{y}$ and $\bar{y}$ are column vectors with $t$th entries $\hat{y}_t$ and $\bar{y}_t$, respectively.

Note that the $(t,t)$th entry of the matrix $\Sigma$ is the variance of $\bar{y}$.

Note further that the Laplacian is of the symmetrized $(A+A^T)/2$ and not of $A$.  For simplicity of notation, we assume from now on that $A$ is symmetric.  If, in practice, an asymmetric $A$ is provided, it can simply be symmetrized.

To find the closed-form solution, we now take the partial derivative of the above objective w.r.t. $\hat{y}$ and equate to zero, obtaining
\begin{align}\label{eq:mta_deriv}
0 &= \frac{1}{T} \left(2 \Sigma^{-1} y^* - 2\Sigma^{-1}\bar{y} \right) + 2\frac{\gamma}{T^2} Ly^* \\
  &=  y^*  - \bar{y} + \frac{\gamma}{T} \Sigma L y^* \nonumber\\
\Leftrightarrow \bar{y} &= \left(I + \frac{\gamma}{T} \Sigma L \right)y^* \nonumber,
\end{align}
which yields the following optimal closed-form solution for $y^*$:
\begin{align}\label{eq:mta_graph}
y^* &= \left(I + \frac{\gamma}{T} \Sigma L\right)^{-1}\bar{y},
\end{align}
as long as the inverse exists, which we will prove in Appendix B.

\section*{Appendix B: Proof of Lemma 1}
\noindent \textbf{Assumptions:} $\gamma \geq 0$, $0 \leq A_{rs} < \infty $ for all $r,s$ and $0 < \frac{\sigma_t^2}{N_t} < \infty$ for all $t$.\\

\noindent \textbf{Lemma 1} \emph{The MTA solution matrix $W = \left(I + \frac{\gamma}{T} \Sigma L\right)^{-1}$}.\\

\begin{proof}
Let $B = W^{-1} = I + \frac{\gamma}{T} \Sigma L$.  The $(t,s)$th entry of $B$ is
\begin{align*}
B_{ts} &=
\begin{cases}
 1 + \frac{\gamma\sigma_t^2}{T N_t}\sum_{s\neq t} A_{ts}  &\text{if }t=s\\
 -\frac{\gamma\sigma_t^2}{T N_t} A_{ts} &\text{if }t\neq s,
\end{cases}
\end{align*}
The Gershgorin disk \citep{Horn} $\mathcal{D}(B_{tt},R_t)$ is the closed disk in $\mathbb{C}$ with center $B_{tt}$ and radius $$R_t = \sum_{s\neq t} |B_{ts}| = \frac{\gamma\sigma_t^2}{T N_t} \sum_{s\neq t} A_{ts}  = B_{tt} - 1.$$
One knows that $B_{tt} \geq 1$ for non-negative $A$ and when $\frac{\gamma\sigma_t^2}{T N_t} \geq 0$, as assumed in the lemma statement.  Also, it is clear that $B_{tt} > R_t$ for all $t$.  Therefore, every Gershgorin disk is contained within the positive half-plane of $\mathbb{C}$, and, by the Gershgorin Circle Theorem \citep{Horn}, the real part of every eigenvalue of matrix $B$ is positive.  Its determinant is therefore positive, and the matrix $B$ is invertible: $W = B^{-1}$.
\end{proof}

\section*{Appendix C: Proof of Theorem 2}
\noindent \textbf{Assumptions:} $\gamma \geq 0$, $0 \leq A_{rs} < \infty $ for all $r,s$ and $0 < \frac{\sigma_t^2}{N_t} < \infty$ for all $t$.\\

Before proving Theorem 2, we will need to prove two more lemmas.

\begin{lemma}
$W$ has all non-negative entries.
\end{lemma}

\begin{proof} By inspection it is clear that $W^{-1} = \left(I + \frac{\gamma}{T} \Sigma L\right)$ is a \emph{Z-matrix}, defined to be a matrix with non-positive off-diagonal entries \citep{Berman79}.  If $W^{-1}$ is a Z-matrix, then the following two statements are true and equivalent: ``the real part of each eigenvalue of $W^{-1}$ is positive'' and ``$W$ exists and $W \geq 0$ (elementwise)'' (Chapter 6, Theorem 2.3, $G_{20}$ and $N_{38}$, \citep{Berman79}).  It has already been proven in Lemma 1 that the real part of every eigenvalue of $W^{-1}$ is positive.  Therefore, $W$ exists and is element-wise non-negative.
\end{proof}

\begin{lemma}
The rows of $W$ sum to 1, i.e. $W\one = \one$.
\end{lemma}

\begin{proof}
As proved in Lemma 1, $W$ exists.  Therefore, one can write:
\begin{align*}
W\one =& \one\\
\Leftrightarrow \one =& W^{-1}\one\\
=& \left(I + \frac{\gamma}{T} \Sigma L\right)\one\\
=& I\one + \frac{\gamma}{T} \Sigma L \one\\
=& \one + \frac{\gamma}{T} \Sigma \zero\\
=& \one,
\end{align*}
where the the third equality is true because the graph Laplacian has rows that sum to zero.  The rows of $W$ therefore sum to 1.
\end{proof}

\noindent \textbf{Theorem}  \emph{The MTA solution matrix $W = \left(I + \frac{\gamma}{T} \Sigma L\right)^{-1}$ is right-stochastic.}\\

\begin{proof}
We know that $W$ exists (from Lemma 1), is entry-wise non-negative (from Lemma 2), and has rows that sum to 1 (from Lemma 3).
\end{proof}

\section*{Appendix D: Constant MTA Derivation}
For the case when $T>2$, analytically specifying a general similarity matrix $A$ that minimizes the risk is intractable.  To address this limitation for arbitrary $T$, we constrain the similarity matrix to be the constant matrix $A = a\one\one^T$, resulting in the following weight matrix:
\begin{equation}\label{eqn:w_cnst_full}
\left(I + \frac{\gamma}{T}\Sigma L(a\one\one^T)\right)^{-1}.
\end{equation}
For a general, asymmetric $A$ there are $T(T-1)$ parameters to estimate.  For the constant $A = a\one\one^T$  only $a$ needs to be estimated ($\gamma$ is set to 1 w.l.o.g.).  It turns out, however, that finding $a^*$ for arbitrary $T$ by minimizing the risk of the estimator $\hat{y} = W^\text{cnst}\bar{y}$ is not tractable, but becomes tractable for a simplified version of (\ref{eqn:w_cnst_full}) where the trace of the covariance replaces the full covariance.  Thus we find $a^*$ as follows
\begin{equation}\label{eqn:w_cnst_simple}
a^*= \argmin_a R\left(\mu,\left(I + \frac{\gamma}{T}\frac{\tr(\Sigma)}{T} L(a\one\one^T)\right)^{-1}\bar{Y}\right).
\end{equation}
and then plug this $a^*$ into (\ref{eqn:w_cnst_full}) to obtain ``constant MTA''.
\begin{equation}\label{eqn:w_cnst_full}
W^\text{cnst} = \left(I + \frac{\gamma}{T}\Sigma L(a^*\one\one^T)\right)^{-1}.
\end{equation}

First, we simplify $W^\text{cnst}$ using the Sherman-Morrison formula:
\begin{align}\label{eqn:w_cnst}
\left(I + \frac{1}{T}\Sigma L(a\one\one^T)\right)^{-1}
&= \left(I + \frac{a}{T}\Sigma L(T I - \one\one^T)\right)^{-1} \nonumber \\
&= \left(I + a \Sigma  - \frac{a}{T} \Sigma \one\one^T\right)^{-1} \nonumber \\
&= (I + a\Sigma)^{-1} + \frac{(I + a \Sigma)^{-1} \frac{a}{T} \Sigma \one\one^T (I + a \Sigma)^{-1}} {1 - \frac{a}{T} \one^T (I + a \Sigma)^{-1}\Sigma\one },
\end{align}
and set $\Sigma = \frac{\tr(\Sigma)}{T} I$ to get $W^\text{smpl}$:
\begin{align*}
W^\text{smpl} &= \frac{1}{1 + a \frac{\tr(\Sigma)}{T}}I + \frac{\frac{1}{1 + a \frac{\tr(\Sigma)}{T}} \frac{a}{T} \frac{\tr(\Sigma)}{T}  \one\one^T \frac{1}{1 + a\frac{\tr(\Sigma)}{T}}} {1 - \frac{a}{T} \one^T \frac{1}{1 + a \frac{\tr(\Sigma)}{T}} \frac{\tr(\Sigma)}{T} \one }\\
&= \frac{1}{a \frac{\tr(\Sigma)}{T} + 1}I + \frac{\frac{ a\frac{\tr(\Sigma)}{T}}{a \frac{\tr(\Sigma)}{T} + 1} \frac{1}{T}  \one\one^T \frac{1}{1 + a \frac{\tr(\Sigma)}{T}}} {1 - \frac{a\frac{\tr(\Sigma)}{T}}{1 + a \frac{\tr(\Sigma)}{T}} }\\
&= \frac{1}{a \frac{\tr(\Sigma)}{T} + 1}I + \frac{ a\frac{\tr(\Sigma)}{T}}{a \frac{\tr(\Sigma)}{T} + 1} \frac{1}{T}  \one\one^T \\
&= \frac{1}{a \frac{\tr(\Sigma)}{T} + 1} \left( I + a\frac{\tr(\Sigma)}{T^2}  \one\one^T \right).
\end{align*}
The risk of $y^* = W^\text{smpl}\bar{y}$ is
\begin{align*}
R(\mu,Y^*) &= \tr(W^\text{smpl}\Sigma (W^\text{smpl})^T) + \mu^T(W^\text{smpl}-I)^T(W^\text{smpl}-I)\mu\\
&= \tr\left(\frac{1}{a \frac{\tr(\Sigma)}{T} + 1} \left( I + a\frac{\tr(\Sigma)}{T^2}  \one\one^T \right)\Sigma I \frac{1}{a \frac{\tr(\Sigma)}{T} + 1} \left( I + a\frac{\tr(\Sigma)}{T^2}  \one\one^T \right)^T\right)
 \\ &\quad + \mu^T\left(\frac{1}{a \frac{\tr(\Sigma)}{T} + 1} \left( I + a\frac{\tr(\Sigma)}{T^2}  \one\one^T \right)-I\right)^T\left(\frac{1}{a \frac{\tr(\Sigma)}{T} + 1} \left( I + a\frac{\tr(\Sigma)}{T^2}  \one\one^T \right)-I\right)\mu\\
&=\frac{1}{(a \frac{\tr(\Sigma)}{T} + 1)^2} \tr\left( \left( I + a\frac{\tr(\Sigma)}{T^2}  \one\one^T \right)\Sigma \left( I + a\frac{\tr(\Sigma)}{T^2}  \one\one^T \right)\right)
 \\ &\quad + \mu^T\left( \frac{-a\frac{\tr(\Sigma)}{T}}{a \frac{\tr(\Sigma)}{T} + 1} I + \frac{a\frac{\tr(\Sigma)}{T}}{a \frac{\tr(\Sigma)}{T} + 1} \frac{1}{T}  \one\one^T \right)^T \left(\frac{-a\frac{\tr(\Sigma)}{T}}{a \frac{\tr(\Sigma)}{T} + 1} I + \frac{a\frac{\tr(\Sigma)}{T}}{a \frac{\tr(\Sigma)}{T} + 1} \frac{1}{T}  \one\one^T \right)\mu\\
&= \frac{1 }{(a \frac{\tr(\Sigma)}{T} + 1)^2} \tr\left(\Sigma + 2a\frac{\tr(\Sigma)}{T^2}\one\one^T\Sigma + a^2\frac{\tr(\Sigma)^2}{T^4}\one\one^T\Sigma\one\one^T \right)
 \\ &\quad + \frac{(a\frac{\tr(\Sigma)}{T})^2}{(a \frac{\tr(\Sigma)}{T} + 1)^2} \mu^TL\left(\frac{1}{T} \one\one^T\right)^TL\left(\frac{1}{T} \one\one^T\right)\mu\\
   &= \frac{\frac{\tr(\Sigma)}{T} }{(a \frac{\tr(\Sigma)}{T} + 1)^2} \left( T + 2a\frac{\tr(\Sigma)}{T} + \left(a \frac{\tr(\Sigma)}{T}\right)^2 \right)
 \\ &\quad +  \frac{(a\frac{\tr(\Sigma)}{T})^2}{(a \frac{\tr(\Sigma)}{T} + 1)^2} \mu^TL\left(\frac{1}{T} \one\one^T\right)^TL\left(\frac{1}{T} \one\one^T\right)\mu\\
\end{align*}
To find the minimum, we take the partial derivative w.r.t. $a$ and set it equal to zero.  Again noting that
$$L\left(\frac{1}{T} \one\one^T\right)^TL\left(\frac{1}{T} \one\one^T\right) = L\left(\frac{1}{T} \one\one^T\right),$$
and omitting some tedious algebra,
\begin{align*}
\frac{\partial}{\partial a^*} R(\mu,Y^*) = 0 &= \frac{2\frac{\tr(\Sigma)}{T}(-T + 1 + a^*\mu^TL\left(\frac{1}{T} \one\one^T\right)\mu)}
{(a^* \frac{\tr(\Sigma)}{T} + 1)^3}\\
\Leftrightarrow a^* &= \frac{T-1}{\mu^TL\left(\frac{1}{T} \one\one^T\right)^TL\left(\frac{1}{T} \one\one^T\right)^T\mu}\\
&= \frac{T-1}{\mu^T L\left(\frac{1}{T}\one\one^T\right) \mu}\\
&= \frac{2}{\frac{1}{T(T-1)}\sum_{r=1}^T\sum_{s=1}^T (\mu_r - \mu_s)^2}.
\end{align*}

\section*{Appendix E: Minimax MTA Derivation}
First, some definitions are in order.
\begin{itemize}
\item An estimator $Y^M$ of $\mu$ which minimizes the maximum risk
$$\inf_{\hat{Y}} \sup_\mu R(\mu,\hat{Y}) = \sup_\mu R(\mu,Y^M),$$
is called a \emph{minimax} estimator.  \\
\item The \emph{average risk} for estimator $\hat{Y}$ is
\begin{equation}\label{eqn:average_risk}
r(\pi,\hat{Y}) = \int R(\mu,\hat{Y}) \pi(\mu) d\mu,
\end{equation}
where $\pi$ is a prior on $\mu$.
\item The estimator that minimizes the average risk is called the \emph{Bayes estimator}
and is written $$Y_\pi = \argmin_{\hat{Y}} r(\pi,\hat{Y}).$$
\item The \emph{Bayes risk} is the risk of the Bayes estimator and is written
\begin{equation}\label{eqn:bayes_risk}
r(\pi,Y_\pi) = \int R(\mu,Y_\pi) \pi(\mu) d\mu.\\
\end{equation}
\item A prior distribution $\pi$ is \emph{least favorable} if $r(\pi,Y_\pi) \geq r(\pi',Y_\pi')$ for all priors $\pi'$.\\
\end{itemize}
To find a minimax MTA, we will need the following theorem and corollary (Theorem 1.4, Chapter 5 \citep{LehmannCasella}).\\ \\

\noindent \textbf{Theorem} \emph{Suppose that $\pi$ is a distribution on the space of $\mu$ such that
$$r(\pi,Y_\pi) = \sup_\mu R(\mu,Y_\pi).$$
Then:
\begin{enumerate}
\item $Y_\pi$ is minimax.
\item If $Y_\pi$ is the \emph{unique} Bayes solution w.r.t. $\pi$ (i.e. if it is the only minimizer of (\ref{eqn:bayes_risk})), then it is the unique minimax estimator.
\item The prior $\pi$ is least favorable.\\
\end{enumerate}}
\noindent  \textbf{Corollary}  \emph{If a Bayes estimator $Y_\pi$ has constant risk, then it is minimax.}\\ \\

The first step in finding a minimax solution for the $T=2$ case is specifying a constraint set for $\mu$ over which a least favorable prior (LFP) can be found.  If no constraint set is used, $\mu_t = \infty$ is the worst case, and leads to a LFP that puts all of its mass on that point.  We will use one of the simplest constraint sets, and constrain each $\mu_t$ to be in the interval $\mu \in [b_l,b_u]^T$, where $b_l \in \R$ and $b_u \in \R$.  To find the LFP we must find the $\mu$ that makes the risk as large as possible.  For $T=2$ and right-stochastic $W$, we have that the $\mu$-dependent term in the (\ref{eqn:risk}) can be written as
$(W-I)^T(W-I) = (W_{12}^2+W_{21}^2)L(\one\one^T)$, and therefore
$$\mu^T(W-I)^T(W-I)\mu = (W_{12}^2+W_{21}^2) (\mu_1-\mu_2)^2,$$
which is clearly maximized by either $\mu_1 = b_u, \mu_2 = b_l$ or $\mu_1 = b_l, \mu_2 = b_u$.  Therefore the LFP is
\begin{align*}
p(\mu) =
\begin{cases}
\frac{1}{2}, & \mbox{if } \mu = (b_l,b_u)  \\
\frac{1}{2}, & \mbox{if } \mu = (b_u,b_l)  \\
0, & \mbox{otherwise.}
\end{cases}\\
\end{align*}
The next step is to \emph{guess} a minimax weight matrix $W^M$ and show that the estimator $Y^M = W^M\bar{Y}$ (i) has constant risk and (ii) is a Bayes solution.  According to the corollary, if both (i) and (ii) hold for the guessed $W^M$, then  $W^M\bar{Y}$ is minimax.  For the $T=2$ case, we guess  $W^M$ to be
\begin{align*}
W^* = \left(I + \frac{2}{T(b_l-b_u)^2}\Sigma L(\one\one^T)\right)^{-1},
\end{align*}
which is just $W^\text{cnst}$ with $a = \frac{2}{(b_l-b_u)^2}$.  This choice of $W$ is not a function of $\mu$ and thus we have shown that (i) the Bayes risk w.r.t the LFP is constant for all $\mu$.  What remains to show is (ii) $W^M$ is indeed the Bayes solution, i.e. it is minimizer of the Bayes risk:
\begin{align}\label{eqn:sum_risks}
&\frac{1}{2}\left( [b_l ~b_u](W-I)^T(W-I) \left[
                                     \begin{array}{c}
                                       b_l \\
                                       b_u \\
                                     \end{array}
                                   \right]
                                   + \tr(W\Sigma W^T)\right) \nonumber\\
+ &\frac{1}{2}\left([b_u ~b_l](W-I)^T(W-I)\left[
                                     \begin{array}{c}
                                       b_u \\
                                       b_l \\
                                     \end{array}
                                   \right]
                                    + \tr(W\Sigma W^T)\right).
\end{align}
Note that this expression is the sum of two convex risks.  We already know (see (7) on page 4 of the NIPS paper) that for $T=2$ the minimizer of the risk
$$[\mu_1 ~\mu_2](W-I)^T(W-I)\left[
                                     \begin{array}{c}
                                       \mu_1 \\
                                       \mu_2 \\
                                     \end{array}
                                   \right] + \tr(W\Sigma W^T)$$
is
$W^* = \left(I + \frac{2}{T(\mu_1-\mu_2)^2}\Sigma L(\one\one^T)\right)^{-1}.$
Thus, the minimizer of
$$[b_l ~b_u](W-I)^T(W-I)\left[
                                     \begin{array}{c}
                                       b_l \\
                                       b_u \\
                                     \end{array}
                                   \right] + \tr(W\Sigma W^T)$$
is
$W^1 = \left(I + \frac{2}{T(b_l-b_u)^2}\Sigma L(\one\one^T)\right)^{-1},$
and the minimizer of
$$[b_u ~b_l](W-I)^T(W-I)\left[
                                     \begin{array}{c}
                                       b_u \\
                                       b_l \\
                                     \end{array}
                                   \right] + \tr(W\Sigma W^T)$$
is
$W^2 = \left(I + \frac{2}{T(b_u-b_l)^2}\Sigma L(\one\one^T)\right)^{-1}.$
Clearly $W^1 = W^2$ which means that the two risks in (\ref{eqn:sum_risks}) are both minimized by the same weight matrix $W^1$, and thus their sum is also minimized by $W^1$.  Therefore
\begin{align}\label{eqn:minimax_mta}
W^M &= \left(I + \frac{2}{T(b_u-b_l)^2}\Sigma L(\one\one^T)\right)^{-1}
\end{align}
as was to be shown.  One can conclude that $W^M$ is minimax over all estimators of the form $W = \left(I + \frac{\gamma}{T}\Sigma L\right)^{-1}$ for $T=2$ using the interval constraint set.

\section*{Appendix F: Proof of Proposition 2}

\noindent \textbf{Proposition 2} \emph{The set of estimators $W\bar{Y}$ where $W$ is of MTA form as per (\ref{eq:mta_general}) is strictly larger than the set of estimators that regularize the single-task estimates as follows:}
\begin{align*}
\hat{Y}_t = \frac{1}{\gamma} \bar{y}_t + \sum_{r=1}^T \alpha_r \bar{Y}_r,
\end{align*}
\emph{where} $\sum_{r=1}^T \alpha_r = 1-\frac{1}{\gamma}$, $0 < \frac{1}{\gamma} \leq 1$, and $\alpha_r \geq 0$, $\forall r$.\\

\begin{proof}
First we will show that estimators $\hat{Y}_t$ can be written in MTA form. Rewriting $\hat{Y}$ in matrix notation:
\begin{align*}
\hat{Y}_t &= \frac{1}{\gamma} \bar{Y}_t + \sum_{t=1}^T \alpha_t \bar{Y}_t\\
\Leftrightarrow \hat{Y} &= \left(\frac{1}{\gamma} I + \one \alpha^T\right)\bar{Y}.
\end{align*}
The goal now is to show that $(\frac{1}{\gamma} I + \one \alpha^T)^{-1}$ has MTA form.  Using the Sherman-Morrison formula, we get
\begin{align*}
\left(\frac{1}{\gamma} I + \one \alpha^T\right)^{-1} &= \gamma I - \frac{\gamma^2 \one \alpha^T}{1 + \gamma \alpha^T\one}\\
&= \gamma I - \gamma \one \alpha^T\\
&= I + (\gamma-1) I - \gamma \one \alpha^T\\
&= I + \gamma\left(1-\frac{1}{\gamma}\right)I - \gamma \one \alpha^T\\
&= I + \gamma L(\one \alpha^T),
\end{align*}
which is a matrix of MTA form with appropriate choices of $\gamma$, $\Sigma$, and $A$ (obtained by visual pattern matching).  Thus, estimators $\hat{Y}_t$ can be written in MTA form:
\begin{equation}\label{eqn:convex_reg}
\hat{Y} = (I + \gamma L(\one \alpha^T))^{-1}.
\end{equation}
By inspection of (\ref{eq:mta_general}), it is clear that not all matrices of the form $(I+\Gamma L(A))^{-1}$ can be written as (\ref{eqn:convex_reg}).  This implies that matrices of MTA form are strictly more general than matrices of the form in (\ref{eqn:convex_reg}).
\end{proof}

\bibliography{ProposalRefs}

\begin{thebibliography}{35}
\providecommand{\natexlab}[1]{#1}
\providecommand{\url}[1]{\texttt{#1}}
\expandafter\ifx\csname urlstyle\endcsname\relax
  \providecommand{\doi}[1]{doi: #1}\else
  \providecommand{\doi}{doi: \begingroup \urlstyle{rm}\Url}\fi

\bibitem[Abernethy et~al.(2009)Abernethy, Bach, Evgeniou, and
  Vert]{Abernethy09}
J.~Abernethy, F.~Bach, T.~Evgeniou, and J.-P. Vert.
\newblock A new approach to collaborative filtering: Operator estimation with
  spectral regularization.
\newblock \emph{Journal Machine Learning Research}, 10, 2009.

\bibitem[Argyriou et~al.(2007)Argyriou, Micchelli, Pontil, and
  Ying]{Argyriou07}
A.~Argyriou, C.~A. Micchelli, M.~Pontil, and Y.~Ying.
\newblock A spectral regularization framework for multi-task structure
  learning.
\newblock In \emph{Advances in Neural Information Processing Systems (NIPS)},
  2007.

\bibitem[Argyriou et~al.(2008)Argyriou, Evgeniou, and Pontil]{Argyriou08}
A.~Argyriou, T.~Evgeniou, and M.~Pontil.
\newblock Convex multi-task feature learning.
\newblock \emph{Machine Learning}, 73\penalty0 (3):\penalty0 243--272, 2008.

\bibitem[Banerjee et~al.(2005)Banerjee, Merugu, Dhillon, and Ghosh]{Banerjee05}
A.~Banerjee, S.~Merugu, I.~S. Dhillon, and J.~Ghosh.
\newblock Clustering with {B}regman divergences.
\newblock \emph{Journal Machine Learning Research}, 6:\penalty0 1705--1749,
  December 2005.

\bibitem[Belkin et~al.(2006)Belkin, Niyogi, and Sindhwani]{Belkin:06}
M.~Belkin, P.~Niyogi, and V.~Sindhwani.
\newblock Manifold regularization: A geometric framework for learning from
  labeled and unlabeled examples.
\newblock \emph{Journal Machine Learning Research}, 7:\penalty0 2399--2434,
  2006.

\bibitem[Berman and Plemmons(1979)]{Berman79}
A.~Berman and R.~J. Plemmons.
\newblock \emph{Nonnegative Matrices in the Mathematical Sciences}.
\newblock Academic Press, 1979.

\bibitem[Bock(1975)]{Bock72}
M.~E. Bock.
\newblock Minimax estimators of the mean of a multivariate normal distribution.
\newblock \emph{The Annals of Statistics}, 3\penalty0 (1), 1975.

\bibitem[Bonilla et~al.(2008)Bonilla, Chai, and Williams]{Bonilla:08}
E.~V. Bonilla, K.~M.~A. Chai, and C.~K.~I. Williams.
\newblock Multi-task {G}aussian process prediction.
\newblock In \emph{Advances in Neural Information Processing Systems (NIPS)}.
  MIT Press, 2008.

\bibitem[Brown et~al.(2004)Brown, Dalton, and Hoyle]{Brown:04}
D.~Brown, J.~Dalton, and H.~Hoyle.
\newblock Spatial forecast methods for terrorist events in urban environments.
\newblock \emph{Lecture Notes in Computer Science}, 3073:\penalty0 426--435,
  2004.

\bibitem[Casella(1985)]{Casella:85}
G.~Casella.
\newblock An introduction to empirical {Bayes} data analysis.
\newblock \emph{The American Statistician}, pages 83--87, 1985.

\bibitem[Chebotarev and Shamis(2006)]{Chebotarev06}
P.~Chebotarev and E.~Shamis.
\newblock The matrix-forest theorem and measuring relations in small social
  groups.
\newblock \emph{Computing Research Repository}, abs/math/0602070, 2006.

\bibitem[Chung(2004)]{Chung94}
F.~R.~K. Chung.
\newblock \emph{Spectral Graph Theory}.
\newblock 2004.

\bibitem[Efron and Morris(1977)]{EfronMorris:77}
B.~Efron and C.~N. Morris.
\newblock Stein's paradox in statistics.
\newblock \emph{Scientific American}, 236\penalty0 (5):\penalty0 119--127,
  1977.

\bibitem[Fouss et~al.(2006)Fouss, Yen, Pirotte, and Saerens]{Fouss06}
F.~Fouss, L.~Yen, A.~Pirotte, and M.~Saerens.
\newblock An experimental investigation of graph kernels on a collaborative
  recommendation task.
\newblock In \emph{ICDM}, pages 863--868, 2006.

\bibitem[Hastie et~al.(2001)Hastie, Tibshirani, and Friedman]{HTF}
T.~Hastie, R.~Tibshirani, and J.~Friedman.
\newblock \emph{The Elements of Statistical Learning}.
\newblock Springer-Verlag, New York, 2001.

\bibitem[Horn and Johnson(1990)]{Horn}
R.~A. Horn and C.~R. Johnson.
\newblock \emph{Matrix Analysis}.
\newblock Cambridge University Press, 1990.
\newblock Corrected reprint of the 1985 original.

\bibitem[Jacob et~al.(2008)Jacob, Bach, and Vert]{Jacob08}
L.~Jacob, F.~Bach, and J.-P. Vert.
\newblock Clustered multi-task learning: A convex formulation.
\newblock In \emph{Advances in Neural Information Processing Systems (NIPS)},
  pages 745--752, 2008.

\bibitem[James and Stein(1961)]{JamesStein:65}
W.~James and C.~Stein.
\newblock Estimation with quadratic loss.
\newblock \emph{Proc. Fourth Berkeley Symposium on Mathematical Statistics and
  Probability}, pages 361–--379, 1961.

\bibitem[Kato et~al.(2008)Kato, Kashima, Sugiyama, and Asai]{KatoNIPS07}
T.~Kato, H.~Kashima, M.~Sugiyama, and K.~Asai.
\newblock Multi-task learning via conic programming.
\newblock In \emph{Advances in Neural Information Processing Systems (NIPS)},
  pages 737--744. 2008.

\bibitem[Lehmann and Casella(1998)]{LehmannCasella}
E.~L. Lehmann and G.~Casella.
\newblock \emph{Theory of Point Estimation}.
\newblock Springer, New York, 1998.

\bibitem[Micchelli and Pontil(2004)]{Micchelli:04}
C.~A. Micchelli and M.~Pontil.
\newblock Kernels for multi--task learning.
\newblock In \emph{Advances in Neural Information Processing Systems (NIPS)},
  2004.

\bibitem[Romano and Siegel(1986)]{Counterexamples}
J.~P. Romano and A.~F. Siegel.
\newblock \emph{Counterexamples in Probability and Statistics}.
\newblock Chapman and Hall, Belmont, CA USA, 1986.

\bibitem[Rue and Held(2005)]{RueHeld}
H.~Rue and L.~Held.
\newblock \emph{Gaussian {M}arkov Random Fields: {T}heory and Applications},
  volume 104 of \emph{Monographs on Statistics and Applied Probability}.
\newblock Chapman \& Hall, London, 2005.

\bibitem[Saerens et~al.(2004)Saerens, Fouss, Yen, and Dupont]{Saerens04}
M.~Saerens, F.~Fouss, L.~Yen, and P.~Dupont.
\newblock The principal components analysis of a graph, and its relationships
  to spectral clustering.
\newblock In \emph{In Proc. Eur. Conf. Machine Learning}, pages 371--383.
  Springer-Verlag, 2004.

\bibitem[Sheldon(2008)]{SheldonNIPS08}
D.~Sheldon.
\newblock Graphical multi-task learning, 2008.
\newblock Advances in Neural Information Processing Systems (NIPS) Workshops.

\bibitem[Sherman and Morrison(1950)]{ShermanMorrison}
Jack Sherman and Winifried~J. Morrison.
\newblock {Adjustment of an Inverse Matrix Corresponding to a Change in One
  Element of a Given Matrix}.
\newblock \emph{Ann. Math. Stat.}, 21:\penalty0 124--127, 1950.

\bibitem[Silverman(1986)]{Silverman}
B.~W. Silverman.
\newblock \emph{Density Estimation for Statistics and Data Analysis}.
\newblock Chapman and Hall, New York, 1986.

\bibitem[Smola and Kondor(2003)]{Smola03}
A.~J. Smola and I.~R. Kondor.
\newblock Kernels and regularization on graphs.
\newblock In \emph{Proceedings of the Annual Conference on Computational
  Learning Theory}, 2003.

\bibitem[Stein(1956)]{Stein:56}
C.~Stein.
\newblock Inadmissibility of the usual estimator for the mean of a multivariate
  distribution.
\newblock \emph{Proc. Third Berkeley Symposium on Mathematical Statistics and
  Probability}, pages 197--206, 1956.

\bibitem[v.~Luxburg(2007)]{Luxburg07}
U.~v.~Luxburg.
\newblock A tutorial on spectral clustering.
\newblock \emph{Computing Research Repository}, abs/0711.0189, 2007.

\bibitem[Xue et~al.(2007)Xue, Liao, Carin, and Krishnapuram]{Xue07}
Y.~Xue, X.~Liao, L.~Carin, and B.~Krishnapuram.
\newblock Multi-task learning for classification with {Dirichlet} process
  priors.
\newblock \emph{Journal Machine Learning Research}, 8:\penalty0 35--63, 2007.

\bibitem[Yajima and Kuo(2006)]{Yajima06}
Y.~Yajima and T.-F. Kuo.
\newblock Efficient formulations for {1-SVM} and their application to
  recommendation tasks.
\newblock \emph{JCP}, 1\penalty0 (3):\penalty0 27--34, 2006.

\bibitem[Zhang and Yeung(2010)]{ZhangUAI10}
Y.~Zhang and D.-Y. Yeung.
\newblock A convex formulation for learning task relationships.
\newblock In \emph{Proc. of the 26th Conference on Uncertainty in Artificial
  Intelligence (UAI)}, 2010.

\bibitem[Zhu(2006)]{Zhu06}
X.~Zhu.
\newblock Semi-supervised learning literature survey, 2006.

\bibitem[Zhu and Lafferty(2005)]{Zhu05}
X.~Zhu and J.~Lafferty.
\newblock Harmonic mixtures: combining mixture models and graph-based methods
  for inductive and scalable semi-supervised learning.
\newblock In \emph{In Proc. Int. Conf. Machine Learning}, pages 1052--1059. ACM
  Press, 2005.

\end{thebibliography}

\end{document}